\documentclass[11pt]{article}
\usepackage[labelfont=bf]{caption}
\usepackage[utf8]{inputenc}
\usepackage{fullpage}
\usepackage[T1]{fontenc}
\usepackage{authblk}
\usepackage{natbib}
\usepackage[colorlinks=true,citecolor=blue, urlcolor=blue, linkcolor=blue]{hyperref}
\usepackage{url}            
\usepackage{booktabs}       
\usepackage{amsfonts}       
\usepackage{nicefrac}       
\usepackage{microtype}
\usepackage{soul}

\usepackage{comment}
\usepackage{lipsum}
\usepackage{setspace}
\usepackage{mathtools}
\usepackage{graphicx}
\usepackage{amssymb}
\usepackage{bbm}
\usepackage{amsthm}
\usepackage{xpatch}
\usepackage{xspace}
\usepackage{mathrsfs}

\usepackage{url}
\usepackage{array}
\usepackage{wrapfig}
\usepackage{multirow}
\usepackage{tabularx}
\usepackage[normalem]{ulem} 
\usepackage{enumerate}
\usepackage{enumitem}

\newcommand\cut[1]{{}}  
\makeatletter
\xpatchcmd{\algorithmic}{\itemsep\z@}{\itemsep=.5ex plus.5pt}{}{}
\makeatother
\usepackage{mathtools}
\mathtoolsset{showonlyrefs}

\newcommand{\sF}{{\mathscr{F}}}
\newcommand{\cA}{{\mathcal{A}}}

\newcommand{\cC}{{\mathcal{C}}}
\newcommand{\cD}{{\mathcal{D}}}
\newcommand{\cE}{{\mathcal{E}}}
\newcommand{\cF}{{\mathcal{F}}}
\newcommand{\cG}{{\mathcal{G}}}

\newcommand{\cK}{{\mathcal{K}}}

\newcommand{\cM}{{\mathcal{M}}}
\newcommand{\cN}{{\mathcal{N}}}

\newcommand{\cS}{{\mathcal{S}}}
\newcommand{\cT}{{\mathcal{T}}}
\newcommand{\cU}{{\mathcal{U}}}

\newcommand{\cX}{{\mathcal{X}}}

\newcommand{\cZ}{{\mathcal{Z}}}

\newcommand{\RR}{\mathbb{R}}

\newcommand{\unif}{{\mathrm{Unif}}}
\newcommand{\poly}{{\mathrm{poly}}} 

\newcommand{\KL}{\mathbf{KL}}

\renewcommand{\iota}{I}

\DeclareMathOperator*{\argmin}{argmin}

\DeclareMathOperator*{\Min}{minimize}
\DeclareMathOperator*{\Max}{maximize}

\newcommand{\bc}{\begin{center}}
\newcommand{\ec}{\end{center}}

\newcommand{\bdm}{\begin{displaymath}}
\newcommand{\edm}{\end{displaymath}}

\newcommand{\beq}{\begin{equation}}
\newcommand{\eeq}{\end{equation}}

\newcommand{\bfl}{\begin{flushleft}}
\newcommand{\efl}{\end{flushleft}}

\newcommand{\bt}{\begin{tabbing}}
\newcommand{\et}{\end{tabbing}}

\newcommand{\beqn}{\begin{align}}
\newcommand{\eeqn}{\end{align}}

\newcommand{\beqs}{\begin{align*}} 
\newcommand{\eeqs}{\end{align*}}  

\newtheorem{theorem}{Theorem}
\numberwithin{theorem}{section}
\newtheorem{assumption}{Assumption}
\numberwithin{assumption}{section}
\newtheorem{definition}{Definition}
\newtheorem{corollary}{Corollary}
\newtheorem{remark}{Remark}

\newtheorem{lemma}{Lemma}
\numberwithin{lemma}{section}

\makeatletter
\renewcommand*\env@matrix[1][\arraystretch]{%
  \edef\arraystretch{#1}%
  \hskip -\arraycolsep
  \let\@ifnextchar\new@ifnextchar
  \array{*\c@MaxMatrixCols c}}
\makeatother

\newcommand{\de}{\text{dim}_E}

\newcommand{\indict}{\mathbf{1}}

\newcommand{\tpi}{\prescript{\diamond}{}{\pi}}

\newcommand{\alg}{ENIAC\xspace}
\newcommand{\alglong}{Exploratory Non-Linear Incremental Actor Critic\xspace}
\newcommand{\pcpg}{PC-PG\xspace}
\newcommand{\algnpgcompute}{ENIAC-NPG-COMPUTE\xspace}
\newcommand{\algnpgsample}{ENIAC-NPG-SAMPLE\xspace}
\newcommand{\algspisample}{ENIAC-SPI-SAMPLE\xspace}
\newcommand{\algspicompute}{ENIAC-SPI-COMPUTE\xspace}

\newenvironment{itemize*}%
{\begin{itemize}[leftmargin=*,topsep=5pt]%
		\setlength{\itemsep}{1pt}%
		\setlength{\parskip}{1pt}}%
	{\end{itemize}}

\newenvironment{enumerate*}%
{\begin{enumerate}[leftmargin=20pt,topsep=5pt]%
		\setlength{\itemsep}{1pt}%
		\setlength{\parskip}{1pt}}%
	{\end{enumerate}}

\usepackage{algorithm}
\usepackage{algpseudocode}

\allowdisplaybreaks
\usepackage{titlesec}

\titleformat*{\section}{\Large\bfseries}
\titleformat*{\subsection}{\large\bfseries}

\title{Provably Correct Optimization and Exploration with Non-linear Policies}
\author[1]{Fei Feng\thanks{fei.feng@math.ucla.edu}}
\author[1]{Wotao Yin\thanks{wotaoyin@math.ucla.edu}}
\author[3]{Alekh Agarwal \thanks{alekha@microsoft.com, corresponding author.}}
\author[2]{Lin F. Yang\thanks{linyang@ee.ucla.edu, corresponding author. LY acknowledges the support from  the Simons Institute for the Theory of Computing at UC-Berkeley (Theory of Reinforcement Learning Program).}}
\affil[1]{Department of Mathematics, UCLA}
\affil[2]{Department of  Electrical and Computer Engineering, UCLA}
\affil[3]{Microsoft Research}
\date{\today}
\begin{document}
\maketitle
\begin{abstract}
Policy optimization methods remain a powerful workhorse in empirical Reinforcement Learning (RL), with a focus on neural policies that can easily reason over complex and continuous state and/or action spaces. Theoretical understanding of strategic exploration in policy-based methods with non-linear function approximation, however, is largely missing. In this paper, we address this question by designing ENIAC, an actor-critic method that allows non-linear function approximation in the critic. We show that under certain assumptions, e.g., a bounded eluder dimension $d$ for the critic class, the learner finds a near-optimal policy in $\widetilde{O}(\poly(d))$ exploration rounds. The method is robust to model misspecification and strictly extends existing works on linear function approximation. We also develop some computational optimizations of our approach with slightly worse statistical guarantees and an empirical adaptation building on existing deep RL tools. We empirically evaluate this adaptation and show that it outperforms prior heuristics inspired by linear methods, establishing the value via correctly reasoning about the agent's uncertainty under non-linear function approximation. 
\end{abstract}

\section{Introduction}\label{sec:intro}


The success of reinforcement learning (RL) in many empirical domains largely relies on developing policy gradient methods with deep neural networks~\citep{schulman2015trust,schulman2017proximal,haarnoja2018soft}. The techniques have a long history in RL~\citep{williams1992simple,sutton1999policy,konda2000actor}. A number of theoretical results study their convergence properties~\citep{kakade2002approximately,scherrer2014local,geist2019theory,abbasi2019politex,agarwal2020theory,russoGlobal} when the agent has access to a distribution over states which is sufficiently exploratory, such as in a generative model. However, unlike their value- or model-based counterparts, the number of policy-based approaches which actively explore and provably find a near-optimal policy remains relatively limited, and restricted to tabular~\citep{shani2020optimistic} and linear function approximation~\citep{cai2019provably,agarwal2020pc} settings. Given this gap between theory and the empirical literature, it is natural to ask how we can design provably sample-efficient policy-based methods for RL that allow the use of general function approximation, such as via neural networks.


In this paper we design an actor-critic method with general function approximation: \emph{Exploratory Non-linear Incremental Actor Critic (ENIAC)}. Our method  follows a similar high-level framework as~\citet{agarwal2020pc}, but with a very different bonus function in order to reason about the uncertainty of our non-linear critic. In each iteration, we use the bonus to learn an optimistic critic, so that optimizing the actor with it results in exploration of the previously unseen parts of the environment. Unlike~\citet{agarwal2020pc}, we allow non-linear function approximation in the critic, which further parameterizes a non-linear policy class through Soft Policy Iteration (SPI)~\citep{even-dar2009online,haarnoja2018soft,geist2019theory,abbasi2019politex,agarwal2020pc} or Natural Policy Gradient (NPG)~\citep{Kakade01,Peters:2008:NA:1352927.1352986,agarwal2020theory} updates. Theoretically, we show that if the critic function class has a bounded \emph{eluder dimension} \cite{russo2013eluder} $d$, then our algorithm outputs a near-optimal policy in $\mathrm{poly}(d)$ number of interactions, with high probability, for both SPI and NPG methods.

Unlike the linear setting studied in \citet{agarwal2020pc}, whose bonus functions can be computed in closed form, the bonus function for a general function class is considerably more complex. Following the recent work on non-linear value-based methods by~\citet{wang2020reinforcement}, the bonus function is based on the \emph{range of values} (or the width function) predicted at a particular state-action pair by the critic function which accurately predicts the observed returns.
Hence, this function characterizes how uncertain we are about a state-action pair given the past observations.
The value-based method in \citet{wang2020reinforcement} relies on solving the value iteration problem using the experience, which introduces dependence issues across different stages of the algorithm. But, we directly use the width function as our exploration bonus and have a simpler sub-sampling design that that in \citet{wang2020reinforcement}.
Under mild assumptions, our bonus function can be computed in a time polynomially depending on the size of the current dataset.
We also provide a heuristic method to compute the bonus functions for neural networks.
Furthermore, all our results are robust to model misspecification and do not require an explicit specification about the transition dynamics as used in \citet{wang2020reinforcement}.

In order to further improve the efficiency, we develop variants of our methods that require no bonus computation in the execution of the actor. The key idea is to replace certain conditional exploration steps triggered by the bonus with a small uniform exploration. Note that this uniform exploration is in addition to the optimistic reasoning, thus different from vanilla $\epsilon$-greedy methods. The bonus is later incorporated while updating the critic, which is a significant optimization in settings where the actor runs in real-time with resource constrained hardware such as robotic platforms~\cite{pan2018agile}, and plays well with existing asynchronous actor-critic updates~\cite{mnih2016asynchronous}.

We complement our theoretical analysis with empirical evaluation on a continuous control domain requiring non-linear function approximation, and show the benefit of using a bonus systematically derived for this setting over prior heuristics from both theoretical and empirical literature.

\paragraph{Related Work} The rich literature on exploration in RL primarily deals with tabular~\citep{kearns2002optimal,brafman2002r,jaksch2010near,jin2018q} and linear~\citep{yang2019reinforcement,jin2019provably} settings with value- or model-based methods. Recent papers~\citep{shani2020optimistic,cai2019provably,agarwal2020pc} have developed policy-based methods also in the same settings. Of these, our work directly builds upon that of~\citet{agarwal2020pc}, extending it to non-linear settings.

For general non-linear function approximation, a series of papers provide statistical guarantees under structural assumptions~\citep{jiang2017contextual,sun2019model,dann2018oracle}, but these do not lend themselves to computationally practical versions. Other works~\citep{du2019provably,misra2020kinematic,agarwal2020flambe} study various latent variable models for non-linear function approximation in model-based settings. The notion of eluder dimension~\citep{russo2013eluder} used in our theory has been previously used to study RL in deterministic settings~\citep{wen2013efficient}. Most related to our work are the recent value-based technique of~\citet{wang2020reinforcement}, which describes a UCB-VI style algorithm with statistical guarantees scaling with eluder dimension and the model-based policy optimization of~\citet{cai2021optimistic}, which incorporates optimism into policy evaluation via building confidence sets of the transition model and uses eluder dimension to define model capacity.
In this paper, we instead study model-free policy-based methods, which provide better robustness to misspecification in theory and are more amenable to practical implementation.

\paragraph{Notation}
Given a set $\cA$, we denote by $|\cA|$ the cardinality of $\cA$, $\Delta(\cA)$ the set of all distributions over $\cA$, and $\unif(\cA)$ the uniform distribution over $\cA$. We use $[n]$ for the index set $\{1, \dots, n\}$. Let $a,b\in\RR^n$. We denote by $a^\top b$ the inner product between $a$ and $b$ and $\|a\|_2$ the Euclidean norm of $a$. Given a matrix $A$, we use $\|A\|_2$ for the spectral norm of $A$. Given a function $f:\cX\rightarrow\RR$ and a finite dataset $\cZ\subset\cX$, we define $\|f\|_{\cZ}:=\sqrt{\sum_{x\in\cZ} f(x)^2}$. We abbreviate Kullback-Leibler divergence to $\KL$ and use $O$ for leading orders in asymptotic upper bounds and $\widetilde{O}$ to hide the polylog factors.

\section{Setting}\label{sec:setting}
\paragraph{Markov Decision Process}
In this paper, we focus on the discounted Markov Decision Process (MDP) with an infinite horizon. We use $\cM$ to represent an MDP. Each MDP is described as a tuple $(\cS, \cA, P, r, \gamma)$, where $\mathcal{S}$ is a possibly infinite state space, $\mathcal{A}$ is a finite action space, $P:\cS\times\cA\rightarrow\Delta(\cS)$ specifies a transition kernel, $r:\cS\times\cA\rightarrow[0,1]$ is a reward function, and $\gamma\in(0,1)$ is a discount factor.

At each time step, the agent observes a state $s \in \mathcal{S}$ and selects an action $a \in \mathcal{A}$ according to a \emph{policy} $\pi:\cS\rightarrow\Delta(\cA)$. The environment then transitions to a new state $s'$ with probability $P(s'|s,a)$ and the agent receives an instant reward $r(s,a)$.

For a policy $\pi$, its $Q$-value function $Q^{\pi}:\cS\times\cA\rightarrow\RR$ is defined as:
\begin{equation}\label{eq:Qdef}
Q^{\pi}(s,a, r) := \mathbb{E}^\pi\bigg[\sum_{t=0}^\infty \gamma^{t} r(s_t ,a_t)|s_0=s, a_0=a\bigg],
\end{equation}
where the expectation is taken over the trajectory following $\pi$. And the value function is $V^\pi(s,r) := \mathbb{E}_{a\sim \pi(\cdot\mid s)}[Q^\pi(s,a,r)]$.
From $V^{\pi}$ and $Q^{\pi}$, 
the advantage function of $\pi$ is: $A^{\pi}(s,a, r) = Q^{\pi}(s,a,r)-V^{\pi}(s, r), \forall s\in\cS, a\in\cA$. We ignore $r$ in $V$, $Q$ or $A$, if it is clear from the context.

Besides value, we are also interested in the distribution induced by a policy. Specifically, we define the discounted state-action distribution $d^{\pi}_{\tilde{s}}(s,a)$ induced by $\pi$ as:
\begin{align}\label{eq:dpisdef}
d^{\pi}_{\tilde{s}}(s,a) = (1-\gamma)\sum_{t=0}^{\infty}  \gamma^t {\Pr}^{ \pi}(s_t=s,a_t=a|s_0=\tilde{s}),
\end{align}
where ${\Pr}^{ \pi}(s_t=s,a_t=a|s_0=\tilde{s})$ is the probability of reaching $(s,a)$ at the $t_{\text{th}}$ step starting from $\tilde{s}$ following $\pi$. Similarly, we define $d^{\pi}_{\tilde{s}, \tilde{a}}(s,a)$ if the agent starts from state $\tilde{s}$ followed by action $\tilde{a}$ and follows $\pi$ thereafter.
For any distribution $\nu\in\Delta(\cS\times\cA)$, we denote by $d^{\pi}_{\nu}(s,a):=\mathbb{E}_{(\tilde{s},\tilde{a})\sim \nu} ~[ d^{\pi}_{(\tilde{s},\tilde{a})}(s,a)]$ and $d^{\pi}_{\nu}(s):=\sum_{a}d^\pi_{\nu}(s,a)$.



Given an initial distribution $\rho\in\Delta(\cS)$, 
we define $V^{\pi}_{\rho}:=\mathbb{E}_{s_0\sim \rho }[V^{\pi}(s_0)]$.
Similarly, if $\nu\in\Delta(\cS\times\cA)$, we define 
$V^{\pi}_{\nu}:=\mathbb{E}_{(s_0,a_0)\sim \nu} [Q^{\pi}(s_0, a_0)]$.
The goal of RL is to find a policy in some policy space $\Pi$ such that its value with respect to an initial distribution $\rho_0$ is maximized, i.e.,
\begin{align}\label{eq:RLGoal}
\Max_{\pi\in\Pi} V^{\pi}_{\rho_0}.
\end{align}
Without loss of generality, we consider the RL problems starting from the unique initial state $s_0$ in the later context. All the results straightforwardly apply to arbitrary $\rho_0$.

\paragraph{Function Class and Policy Space}
Let $\cF:=\{f:\cS\times\cA\rightarrow\RR\}$ be a general function class, e.g., neural networks. We denote by $\Pi_{\cF}:=\{\pi_f, f\in\cF\}$ a policy space induced by applying the softmax transform to functions in $\cF$, i.e.,
\begin{align}\label{eq:softmax}
        \pi_{f}(a|s) = \frac{\exp(f(s,a))}{\sum_{a'\in\cA} \exp(f(s,a'))}.
\end{align}
For the ease of presentation, we assume there exists a function $f\in \cF$ such that, for all $s$, $\pi_{f}(\cdot|s)$ is a uniform distribution\footnote{This requirement is not strict, our algorithms and analysis apply for any distribution that are supported on all actions.} on $\cA$.
Given $\cF$, we define its \emph{function-difference} class $\Delta\cF:=\{\Delta f~\vert~ \Delta f = f-f',~ f,f'\in\cF\}$ and the width function on $\Delta\cF$ as: 
\begin{align}
\label{eqn:width}
    \hspace{-0.5cm} w(\Delta\cF, s,a) := \sup_{\Delta f\in\Delta\cF}  \Delta f(s,a),~~\forall (s,a)\in\cS\times\cA.
\end{align}
Note that our width is defined on the function difference class instead of the original function class $\cF$, where the latter is adopted in \cite{russo2013eluder} and \cite{wang2020reinforcement}. These two formulations are essentially equivalent. 

If $\cF$ can be smoothly parameterized by $\theta\in\RR^d$, we further introduce the (centered) \emph{tangent class} of $\cF_{\theta}$ as:
\begin{align}\label{eq:tangent}
    \cG_\cF:=\{ g^{u}_ {\theta}~|~g^{u}_{\theta}(s,a):=u^\top\nabla_{\theta}\log &\pi_{f_{\theta}}(s,a), u\in\cU, f_{\theta}\in\cF\},
\end{align}
where $\cU\subset\RR^d$ is some bounded parameter space.
We define the function-difference class $\Delta\cG_{\cF}$ and the width function $ w(\Delta\cG_{\cF}, s,a)$ for $\cG_{\cF}$ accordingly.

Next, given a function class $\cF$, we consider RL on the induced policy space $\Pi_{\cF}$. If $\cF$ is non-smooth, we apply SPI as the policy optimization routine while approximating $Q$-values with $\cF$; if $\cF$ is smoothly parameterized by $\theta$, we can alternatively apply NPG for policy optimization and use $\cG_{\cF}$ to approximate advantage functions. The corresponding function-difference classes are used to design bonus functions and guide exploration.

\section{Algorithms}\label{sec:algo}

\begin{algorithm}[t]
  \caption{\alglong(\alg)
  }
  \label{alg:ENIAC}
  \begin{algorithmic}[1]
  \State \textbf{Input:} Function class $\cF$. 
  \State \textbf{Hyperparameters:} $N>0$, $K>0$, $\beta>0$, $\alpha\in(0,1)$.
  \State For all $s\in\cS$, initialize $\pi^1(\cdot|s)=\unif(\cA)$.\label{line:init}
  \State Let experience buffer $\cZ^0=\emptyset$.\label{eniac:init}
  \For {$n=1$ \textbf{to} $N$}
  \State Generate $K$ samples: $\{s_i,a_i\}_{i=1}^K\sim d^{\pi^n}_{s_0}$;\label{line:NPG-collect}
  \State Merge training set: $\cZ^{n}\gets \cZ^{n-1}\cup\{s_i,a_i\}_{i=1}^K$;
  \State Let $\rho^n_{\text{cov}}:=\unif(d^{\pi^1}_{s_0}, \dots, d^{\pi^n}_{s_0})$; \label{line:cover} 
  \State Define a bonus function $b^n$ using \eqref{eqn:bonus-sample} or \eqref{eqn:bonus-compute}; \label{eniac:bonus}
  \State Update the policy using Algorithm~\ref{alg:NPG}: $\pi^{n+1}\gets$ Policy Update($\rho^n_{\text{cov}}$, $b^n$, $\alpha$).\label{eniac:update}
  \EndFor
  \State\textbf{Output:}  $\unif(\pi_2, \pi_3, \dots, \pi_{N+1})$
  \end{algorithmic}
\end{algorithm}

\begin{algorithm}[t]
  \caption{Policy Update}
  \label{alg:NPG}
  \begin{algorithmic}[1]
  \State \textbf{Input:} Fitting distribution $\rho$, bonus function $b$, $\alpha$.
  \State \textbf{Hyperparameters:} $T>0$, $M>0$, $\eta>0$.
  \State Initialize $\pi_0$ using \eqref{eq:init-sample} or \eqref{eq:init-compute}.
  \For {$t=0$ \textbf{to} $T-1$}
  \State Generate $M$ samples from  $\rho$ using \eqref{eqn:spi-data} or \eqref{eqn:npg-data}; \label{line:Msample}
  \State Fit critic to the $M$ samples using \eqref{eqn:spi-critic-fit} or \eqref{eqn:npg-critic-fit};
  \State Actor update using 
  \eqref{eqn:SPI-actor-sample}, \eqref{eqn:SPI-actor-compute}, or \eqref{eqn:npg-actor-sample} to obtain $\pi_{t+1}$;
  \EndFor
  \State \textbf{Output:} 
  $\unif(\pi_0, \pi_1, \dots, \pi_{T-1})$
  \end{algorithmic}
\end{algorithm}

In this section, we describe our algorithm, \emph{\alglong} (\alg),  which takes a function class $\cF$ and interacts with an RL environment to learn a good policy. The formal pseudo-code is presented in Algorithm~\ref{alg:ENIAC}. 
We explain the high-level design and steps in the algorithm in this section, before giving our main results in the next section.

\subsection{High-level Framework}
At a high-level, \alg solves a series of policy optimization problems in a sequence of carefully designed MDPs. Each MDP is based on the original MDP, but differs in the choice of an \emph{initial state distribution} and a \emph{reward bonus}. We use them to induce optimistic bias to encourage exploration. Through the steps of the algorithm, the initial distribution gains coverage, while the bonus shrinks so that good policies in the modified MDPs eventually yield good policies in the original MDP as well.

A key challenge in large state spaces is to quantify the notion of \emph{state coverage}, which we define using the function class $\cF$. We say a distribution $\rho_{\mathrm{cov}}$ provides a good coverage if any function $f\in \cF$ that has a small prediction error on data sampled from  $\rho_{\mathrm{cov}}$ also has a small prediction error under the state distribution $d^{\pi}$ for any other policy $\pi$. In tabular settings, this requires $\rho_{\text{cov}}$ to visit each state, while coverage in the feature space suffices for linear MDPs~\cite{jin2019provably,yang2019reinforcement}.

In \alg, we construct such a covering distribution $\rho_{\text{cov}}$ iteratively, starting from the state distribution of a uniform policy and augmenting it gradually as new policies visit previously unexplored parts of the MDP. Concretely, we maintain a \emph{policy cover} $\{\pi^1, \pi^2, \ldots\}$, which initially contains  only a random policy, $\pi^1$, (Line~\ref{line:init} of Algorithm~\ref{alg:ENIAC}).
At iteration $n$, the algorithm lets $\rho^n_{\text{cov}}$ be a uniform mixture of $\{d^{\pi^1}_{s_0}, d^{\pi^2}_{s_0}, \ldots, d^{\pi^n}_{s_0}\}$ (line~\ref{line:cover}).

Having obtained the cover, we move on to induce the reward bonus by collecting a dataset of trajectories from $\rho^n_{\text{cov}}$ (line~\ref{line:NPG-collect}).\footnote{In the Algorithm~\ref{alg:ENIAC}, only $\pi^n$ is rolled out as the samples can be combined with historical data to form samples from $\rho^n_{\mathrm{cov}}$.}  
These collected trajectories are used to identify a set $\cK^n$ of  state-action pairs covered by $\rho^n_{\text{cov}}$: any functions $f,g\in \cF$ that are close under $\rho^n_{\text{cov}}$ also approximately agree with each other for all $s,a\in \cK^n$. We then create a \emph{reward bonus}, $b^n$  (Line~\ref{eniac:bonus}, formally defined later), toward encouraging explorations outside the set $\cK^n$. 

Finally, taking $\rho^n_{\text{cov}}$ as the initial distribution and the bonus augmented reward $r+b^n$ as the reward function, we find a policy $\pi$ that approximately maximizes $V^\pi_{\rho^n_{\text{cov}}}(r+b^n)$ (line~\ref{eniac:update}). It can be shown that this policy either \emph{explores} by reaching new parts of the MDP or \emph{exploits} toward identifying a near optimal policy. We then add this policy to our cover and proceed to the next epoch of the algorithm .

Within this high-level framework, different choices of the policy update and corresponding bonus functions induce different concrete variants of Algorithm~\ref{alg:ENIAC}. We describe these choices below.

\subsection{Policy Optimization}
In this section, we describe our policy optimization approach, given a policy cover $\rho$ and a reward bonus $b$. We drop the dependence on epoch $n$ for brevity, and recall that the goal is to optimize $V^\pi_\rho(r+b)$. We present two different actor critic style optimization approaches: Soft Policy Iteration (SPI) and Natural Policy Gradient (NPG), which offer differing tradeoffs in generality and practical implementation. SPI is amenable to arbitrary class $\cF$, while NPG requires second-order smoothness. On the other hand, NPG induces fully convex critic objective for any class $\cF$, and is closer to popular optimization methods like TRPO, PPO and SAC. Our presentation of both these methods is adapted from~\cite{agarwal2020theory}, and we describe the overall outline of these approaches in Algorithm~\ref{alg:NPG}, with the specific update rules included in the rest of this section.

For each approach, we provide a \emph{sample-friendly} version and a \emph{computation-friendly} version for updating the policy. 
The two versions of updating methods only differ in the initialization and actor updating steps.
The computation-friendly version provides a policy that can be executed efficiently while being played. The sample-friendly version requires to compute the bonus function during policy execution but saves samples up to $\poly(|\cA|)$ factors.
We now describe these procedures in more details. 

\subsubsection{Policy Initialization}
For both SPI and NPG approaches, we use the following methods to initialize the policy.

\textbf{Sample-friendly initialization.}
Given bonus $b$, we define $\cK:=\{(s,a)~|~b(s,a)=0\}$. We abuse the notation $s\in\cK$ if $b(s,a)=0,~ \forall a\in\cA$. 
We initialize the policy as follows.
\begin{align}
\label{eq:init-sample}
    \pi_0(\cdot|s) = \begin{cases}
      \unif(\cA) &  s\in\cK;\\
      \unif(\{a\in\cA :(s,a)\notin \cK\}) & \text{o.w.}
      \end{cases}
\end{align}
Here the policy selects actions uniformly for states where all actions have been well-explored under $\rho$ and only plays actions that are not well-covered in other states.
Note that such a  policy can be represented by $b$ and a function $f\in \cF$.

\textbf{Computation-friendly initialization.}
The computation-friendly method does not recompute the set $\cK$ and initialize the policy to be purely random, i.e.,
\begin{align}
\label{eq:init-compute}
    \pi_0(\cdot|s) =
      \unif(\cA), ~\forall s\in\cS.
\end{align}

\subsubsection{SPI Policy Update}

For each iteration, $t$, we first generate $M$ (some parameter to be determined) $Q$-value samples with the input distribution $\rho$ as the initial distribution:
\begin{align}
\label{eqn:spi-data}
\{s_i,a_i,
  \widehat{Q}^{\pi_{t}}(s_i,a_i,r+b)\}_{i=1}^{M}, ~(s_i, a_i)\sim \rho,
 \end{align}
 where $\widehat{Q}^{\pi_{t}}$ is an unbiased estimator of $Q^{\pi_{t}}$  (see, e.g., Algorithm~\ref{alg:estimator} in the Appendix).
Then we fit a critic to the above samples by setting $f_t$ as a solution of :
 \begin{align}
 \label{eqn:spi-critic-fit}
     \Min_{f\in\cF} \sum_{i=1}^M\big(\widehat{Q}^{\pi_{t}}(s_i,a_i,r+b)-b(s_i,a_i)- f(s_i,a_i)\big)^2.
 \end{align}
Here we offset the fitting with the initial bonus to maintain consistency with linear function approximation results, where a non-linear bonus introduces an approximation error~\citep{jin2019provably,agarwal2020pc}.
Note that for the SPI, we do not require $f$ to be differentiable.

Based on the critic, we update the actor to a new policy. There are two update versions: one is more sample-efficient, the other is more computational-convenient.

\textbf{Sample-friendly version.} For this version, we only update the policy on states $s\in\cK$ since our critic is unreliable elsewhere. For $s\notin\cK$, we keep exploring previously unknown actions by simply sticking to the initial policy. Then the policy update rule is:
 \begin{align}
      \pi_{t+1}(a|s) \propto \pi_{t}(a|s)\exp\big(\eta f_{t}(s,a)\indict{\{s\in\cK\}}\big),\label{eqn:SPI-actor-sample}
  \end{align}
 where $\eta>0$ is a step size to be specified. Note that since $b(s,a)\equiv 0$ for $s\in\cK$, Equation \eqref{eqn:SPI-actor-sample} is equivalent to $\pi_{t+1}(a|s) \propto \pi_{t}(a|s)\exp\big(\eta (f_{t}(s,a)+b(s,a))\indict{\{s\in\cK\}}\big)$ where the initial bonus is added back.
 
\textbf{Computation-friendly version.} For this version, we remove the indicator function while allowing some probability of uniform exploration:
  \begin{align}
       &\pi'_{t+1}(a|s)\propto \pi'_{t}(a|s)\exp\big(\eta f_{t}(s,a)\big),~~~\pi_{t+1}(\cdot|s)= (1-\alpha)\cdot \pi'_{t+1} + \alpha \cdot \unif(\cA).
      \label{eqn:SPI-actor-compute}
  \end{align}
  Above, $\{\pi_{t}'\}$ is an auxiliary sequence of policies initialized as $\pi_{0}$  and $\alpha>0$. Note that for $s\in\cK$, since $b(s,a)\equiv 0$ we still have $\pi'_{t+1}(a|s)\propto \pi'_{t}(a|s)\exp\big(\eta (f_{t}(s,a)+b(s,a))\big)$, i.e., the offset initial bonus is added back. Thus, compared with Equation \eqref{eqn:SPI-actor-sample}, Equation \eqref{eqn:SPI-actor-compute} differs at: 1. $\alpha$-probability random exploration for $s\in\cK$; 2. update policy for $s\notin\cK$ with a possibly not correct value (if $b(s,a)\neq 0$) but guarantees at least $\alpha$-probability random exploration. Such a change creates a polynomial scaling with $|\cA|$ in the sample complexity but saves us from computing bonuses during policy execution which is required by the sample-friendly version.


\subsubsection{NPG Policy Update}
NPG update shares the same structure as that for SPI. Recall that now the function class $\cF$ is smoothly parameterized by $\theta$.
At each iteration $t$, we first generate $M$ (some parameter to be determined) advantage samples from the input distribution $\rho$,
\begin{align}
\label{eqn:npg-data}
\{s_i,a_i,
  \widehat{A}^{\pi_{t}}(s_i,a_i,r+b)\}_{i=1}^{M}, ~(s_i, a_i)\sim \rho
 \end{align}
 where $\widehat{A}^{\pi_{t}}$ is an unbiased estimator of $A^{\pi_{t}}$  (using Algorithm~\ref{alg:estimator}). We define $\bar b_t(s,a) := b(s,a) - \mathbb{E}_{a\sim \pi_t(\cdot|s)}[b(s,a)]$ as a centered version of the original bonus and $g_t(s,a) := \nabla_{\theta} \log{\pi_{f_{\theta_t}}}(s,a)$ to be the tangent features at $\theta_t$.
We then fit a critic to the bonus offset target $\widehat{A}^{\pi_{t}} - \bar b_t$ by setting $u_t$ as a solution of:
 \begin{align}
 \label{eqn:npg-critic-fit}
  \hspace*{-0.5cm}
 \Min_{u\in\cU} \sum_{i=1}^M\big(\widehat{A}^{\pi_{t}}(s_i,a_i,r+b) - \bar b_t(s_i,a_i)- u^\top 
  g_t(s_i,a_i)
  \big)^2.
 \end{align}
Compared to SPI, a big advantage is that the above critic objective is a linear regression problem, for which any off-the-shelf solver can be used, even with a large number of samples in high dimensions. 

With the critic, we update the actor to generate a new policy as below.

\textbf{Sample-friendly version.} Similar to the sample-friendly version of SPI, we only update the policy on $s\in\cK$ as:
 \begin{align}
&\theta_{t+1} = \theta_{t} + \eta u_t,~~~\pi_{t+1}(a|s) \propto \exp(f_{\theta_{t+1}}(s,a) \indict\{s\in\cK\}),\quad
\label{eqn:npg-actor-sample}
 \end{align}
where $\eta>0$ is a step size to be specified. 
 
We omit the details of the computation-friendly version, which is obtained similar to the counterpart in SPI.
\subsection{Bonus Function}\label{sec:bonus}
In this section, we describe the bonus computation given a dataset $\cZ^n$ generated from some covering distribution $\rho_{\mathrm{cov}}$.
As described in previous subsections, the bonus assigns value $0$ to state-action pairs that are well-covered by $\rho_{\mathrm{cov}}$  and a large value elsewhere.
To measure the coverage, we use a width function (defined in Equation~\eqref{eqn:width}) dependent on $\cZ^n$.
The bonus differs slightly for the SPI and NPG updates since SPI uses $\cF$ for critic fit while NPG use $\cG_{\cF}$.
Specifically, for the sample-friendly version, we take the following bonus function
\begin{align} \label{eqn:bonus-sample}
b^n(s,a) = \indict\{w(\widetilde{\cF}^n,s,a)\geq \beta\}\cdot \frac{1}{1-\gamma},
\end{align}
where for SPI,
\begin{align}
 \widetilde{\cF}^n := \left\{\Delta f\in\Delta\cF~\vert~ \|\Delta f\|_{\cZ^n}\leq \epsilon\right\}
  \label{eqn:spi_diff_class}
\end{align}
and for NPG, 
\begin{align}
 \widetilde{\cF}^n := \left\{\Delta g\in\Delta\cG_{\cF}~\vert~ \|\Delta g\|_{\cZ^n}\leq \epsilon\right\}
  \label{eqn:npg_tangent_diff_class}
\end{align}
with $\cG_{\cF}$ being the tangent class defined in Equation~\eqref{eq:tangent}. 
Here $\beta, \epsilon$ are positive parameters to be determined.
For the computation-friendly version, we scale up the bonus by a factor of $|\cA|/\alpha$ to encourage more exploration, i.e.,
\begin{align}b^n(s,a):=\indict\{w(\widetilde{\cF}^n,s,a)\geq \beta\}\cdot \frac{|\cA|}{(1-\gamma)\alpha}.
\label{eqn:bonus-compute}
\end{align}

\begin{remark}
The bonus can be computed efficiently by reducing the width computation to regression \cite{foster2018practical}.
We can additionally improve the computational efficiency using the sensitivity sampling technique developed in \citet{wang2020reinforcement}, which significantly subsamples the dataset $\cZ$.
We omit the details for brevity.
For neural networks, we provide a heuristic to approximate the bonus in Section~\ref{sec:exp}. 
\end{remark}

\subsection{Algorithm Name Conventions}
Since Algorithm~\ref{alg:ENIAC} provides different options for sub-routines, we specify different names for them as below.
\begin{itemize}[nosep,noitemsep,leftmargin=*]
    \item \algspisample (\alg with sample-friendly SPI update): initialize with \eqref{eq:init-sample}, collect data with \eqref{eqn:spi-data}, fit critic using \eqref{eqn:spi-critic-fit}, and update actor using \eqref{eqn:SPI-actor-sample};
    \item \algspicompute (\alg with computation-friendly SPI update): initialize with \eqref{eq:init-compute}, collect data with \eqref{eqn:spi-data}, fit critic using \eqref{eqn:spi-critic-fit}, and update actor using \eqref{eqn:SPI-actor-compute};
    
    \item \algnpgsample (\alg with sample-friendly NPG update): initialize with \eqref{eq:init-sample}, collect data with \eqref{eqn:npg-data}, fit critic using \eqref{eqn:npg-critic-fit}, and update actor using \eqref{eqn:npg-actor-sample};
    
    \item \algnpgcompute (\alg with computation-friendly NPG update): initialize with \eqref{eq:init-compute}, collect data with \eqref{eqn:npg-data}, fit critic using \eqref{eqn:npg-critic-fit}, and update actor using a similar fashion as~\eqref{eqn:SPI-actor-compute} modified  from~\eqref{eqn:npg-actor-sample}.
\end{itemize}

\section{Theory}\label{sec:theory}
In this section, we provide convergence results of \alg with both the SPI and NPG options in the update rule. We only present the main theorems and defer all proofs to the Appendix. We use superscript $n$ for the $n$-th epoch in Algorithm \ref{alg:ENIAC} and the subscript $t$ for the $t$-th iteration in Algorithm \ref{alg:NPG}. For example, $\pi_t^n$ is the output policy of the $t$-th iteration in the $n$-th epoch.

The sample complexities of our algorithms depend on the complexity of the function class for critic fit (and also the policy, implicitly). To measure the latter, we adopt the notion of eluder dimension which is first introduced in \cite{russo2013eluder}.
\begin{definition}[Eluder Dimension]\label{def:eluder}
Given a class $\cF$, $\epsilon\geq 0$, and $\cZ:=\{(s_i,a_i)\}_{i=1}^n$ be a sequence of state-action pairs.
\begin{itemize*}
    \item A state-action pair $(s,a)$ is $\epsilon$-dependent on $\cZ$ with respect to $\cF$ if any $f,f'\in\cF$ satisfying $\|f-f'\|_{\cZ}:=\sqrt{\sum_{(s',a')\in\cZ}(f(s',a')-f'(s',a'))^2}\leq\epsilon$ also satisfy $|f(s,a)-f'(s,a)|\leq \epsilon$.
    \item An $(s,a)$ is $\epsilon$-independent of $\cZ$ with respect to $\cF$ if $(s,a)$ is not $\epsilon$-dependent on $\cZ$.
    \item The $\epsilon$-eluder dimension $\de(\cF, \epsilon)$ of a function class $\cF$ is the length of the longest sequence of elements in $\cS\times\cA$ such that, for some $\epsilon'\geq\epsilon$, every element is $\epsilon'$-independent of its predecessors.
\end{itemize*}
\end{definition}
It is well known~\citep{russo2013eluder} that if $f(z) = g(w^Tz)$, where $z \in \mathbb{R}^d$, and $g$ is a smooth and strongly monotone link function, then the eluder dimension of $\cF$ is $O(d)$, where the additional constants depend on the properties of $g$. In particular, it is at most $d$ for linear functions, and hence provides a strict generalization of results for linear function approximation.

Based on this measure, we now present our main results for the SPI and NPG in the following subsections.
For the sake of presentation, we provide the complexity bounds for \algspisample and \algnpgsample. The analysis for the rest of the algorithm options is similar and will be provided in the Appendix.

\subsection{Main Results for \alg-SPI}\label{sec:PEGA-SPI}
At a high-level, there are two main sources of suboptimality. First is the error in the critic fitting, which further consists of both the \emph{estimation error} due to fitting with finite samples, as well as an \emph{approximation error} due to approximating the $Q$ function from a restricted function class $\cF$. Second, we have the suboptimality of the policy in solving the induced optimistic MDPs at each step. The latter is handled using standard arguments from the policy optimization literature (e.g.~\citep{abbasi2019politex,agarwal2020theory}), while the former necessitates certain assumptions on the representability of the class $\cF$. To this end, we begin with a closedness assumption on $\cF$. For brevity, given a policy $\pi$ we denote by
\begin{align}\label{eq:Tpi}
\cT^\pi f(s,a) := \mathbb{E}^\pi[r(s,a) + \gamma f(s',a') | s,a].
\end{align}

\begin{assumption}[$\cF$-closedness]
For all $\pi\in \{\cS\to\Delta(\cA)\}$ and $g~:~\cS\times\cA\to[0,\tfrac{2}{(1-\gamma)^2}]$, we have $\cT^{\pi} g \in \cF$.
\label{ass:eval_close}
\end{assumption}
Assumption \ref{ass:eval_close} is a policy evaluation analog of a similar assumption in~\cite{wang2020reinforcement}. For linear $f$, the assumption always holds if the MDP is a linear MDP~\citep{jin2019provably} under the same features. We also impose regularity and finite cover assumptions on $\cF$.

\begin{assumption}[Regularity]\label{ass:W}
We assume that $\max(\sup_{f\in\cF}\|f\|_{\infty}, \frac{1}{1-\gamma})\leq W$.
\end{assumption}

\begin{assumption}[$\epsilon$-cover]\label{ass:spi-cover}
For any $\epsilon>0$, there exists an $\epsilon$-cover $\cC(\cF,\epsilon)\subseteq\cF$ with size $|\cC(\cF,\epsilon)|\leq \cN(\cF, \epsilon)$ such that for any $f\in\cF$, there exists $f'\in\cC(\cF, \epsilon)$ with $\|f-f'\|_\infty\leq \epsilon$.
\end{assumption}

With the above assumptions, we have the following sample complexity result for \alg-SPI-SAMPLE.
\begin{theorem}[Sample Complexity of \algspisample]
\label{thm:SPI-Sample}
Let $\delta\in(0, 1)$ and $\varepsilon\in(0, 1/(1-\gamma))$. Suppose
Assumptions \ref{ass:eval_close}, \ref{ass:W}, and~\ref{ass:spi-cover} hold. With proper hyperparameters, \alg-SPI-SAMPLE returns a policy $\pi$ satisfying \mbox{$V^\pi \geq V^{\pi^\star} - \varepsilon$} with probability at least $1-\delta$ after taking at most
\begin{align}
  \widetilde{O}\Big(\frac{W^8\cdot\big(\de(\cF, \beta)\big)^2\cdot\big(\log(\cN(\cF, \epsilon'))\big)^2}{\varepsilon^{8}(1-\gamma)^{8}}\Big)
\end{align}
samples,
where $\beta=\varepsilon(1-\gamma)/2$ and $\epsilon'=\poly(\varepsilon, \gamma, 1/W, 1/\de(\cF, \beta))$\footnote{The formal definition of $\epsilon'$ can be found in Theorem \ref{thm:SPI-SAMPLE-Detail}}.
\end{theorem}

One of the technical challenges of proving this theorem is to establish an eluder dimension upper bound on the sum of the error sequence. Unlike that in \cite{russo2013eluder} and \cite{wang2020reinforcement}, who apply the eluder dimension argument directly to a sequence of data points, we prove a new bound that applies to the sum of expectations over a sequence of distributions.
This bound is then carefully combined with the augmented MDP argument in \cite{agarwal2020pc} to establish our exploration guarantee. The proof details are displayed in Appendix \ref{app:SPI}.
We now make a few remarks about the result.

\paragraph{Linear case.}
When $f(s,a) = u^T\phi(s,a)$ with $u,\phi(s,a) \in \mathbb{R}^d$, $\de(\cF, \beta) = O(d)$. Our result improves that of \cite{agarwal2020pc} by using Bernstein concentration inequality to bound the generalization error. If Hoeffding inequality is used instead, our complexity will match that of \cite{agarwal2020pc}, thereby strictly generalizing their work to the non-linear setting.

\paragraph{Model misspecification}
Like the linear case, \alg-SPI (both SAMPLE and COMPUTE) is robust to the failure of Assumption~\ref{ass:eval_close}. In Appendix \ref{app:SPI}, we provide a \emph{bounded transfer error} assumption, similar to that of~\citet{agarwal2020pc}, under which our guarantees hold up to an approximation error term. Informally, this condition demands that for any policy $\pi_t$, the best value function estimator $f_t^*$ computed from \emph{on-policy samples} also achieves a small approximation error for $Q^{\pi_t}$ under the distribution $d^{\pi^*}_{s_0}$. A formal version is presented in the Appendix.

\paragraph{Comparison to value-based methods.}
Like the comparison between LSVI-UCB and PC-PG in the linear case, our results have a poorer scaling with problem and accuracy parameters than the related work of~\citet{wang2020reinforcement}. However, they are robust to a milder notion of model misspecification as stated above and readily lend themselves to practical implementations as our experiments demonstrate.

\paragraph{Sample complexity of \algspicompute.} As remarked earlier, a key computational bottleneck in our approach is the need to compute the bonus while executing our policies. In Appendix \ref{sec:SPI-COMPUTE-PROOF} we analyze \algspicompute, which avoids this overhead and admits a
\begin{align}
  \widetilde{O}\Big(\frac{W^{10}\cdot|\cA|^2\cdot\big(\de(\cF, \beta)\big)^2\cdot \big(\log(\cN(\cF,  \epsilon'))\big)^2}{\varepsilon^{10}(1-\gamma)^{10}}\Big)
\end{align}
sample complexity under the same assumptions. The worse sample complexity of \algspicompute arises from: 1. the uniform sampling over all actions instead of targeted randomization only over unknown actions for exploration; 2. $\alpha$-probability uniform exploration even on known states.  

\subsection{Main Results for \alg-NPG}\label{sec:PEGA-NPG}
The results for \alg-NPG are qualitatively similar to those for \alg-SPI. However, there are differences in details as we fit the advantage function using the \emph{tangent class}  $\cG_{\cF}$ now, and this also necessitates some changes to the underlying assumptions regarding closure for Bellman operators and other regularity assumptions. We start with the former, and recall the definition of the tangent class $\cG_\cF$ in Equation~\eqref{eq:tangent}. For a particular function $f\in\cF$, we further use $\cG_f\subseteq \cG_\cF$ to denote the subset of linear functions induced by the features $\nabla_{\theta} \log \pi_{f_\theta}$.

\begin{assumption}[$\cG_f$-closedness]
For any $f\in\cF$, let $\pi_f(a|s)\propto\exp(f(s,a))$. For any measurable set $\cK\in\cS\times\cA$ and $g~:~\cS\times\cA\to[0,\frac{4}{(1-\gamma)^2}]$, we have $\cT^{\pi_{f,\cK}} g-\mathbb{E}_{\pi_{f,\cK}}[\cT^{\pi_{f,\cK}}g] \in \cG_f$, where
\begin{align}
    \pi_{f,\cK}(\cdot|s) =\begin{cases} \pi_f(\cdot|s),  &\text{if for all } a\in\cA, (s,a)\in\cK\\
     \unif(\{a| (s,a)\notin \cK\}), & o.w.
     \end{cases}
\end{align}
and the operator $\cT$ is defined in Equation \eqref{eq:Tpi}.
\label{ass:eval_close_npg}
\end{assumption}
One may notice that the policy $\pi_{f, \cK}$ complies with our actor update in \eqref{eqn:npg-actor-sample} since $b=0$ for $s\in\cK$. We also impose regularity and finite cover assumptions on $\cG_\cF$ as below.
\begin{assumption}[Regularity]\label{ass:NPG-REG}
We assume that $\|u\|_2\leq B$ for all $u\in \cU\subset \RR^d$, and $f_\theta$ is twice differentiable for all $f_\theta \in \cF$, and further satisfies:
\begin{equation*}
    \|f_{\theta}\|_{\infty}\leq W, ~~\|\nabla f_{\theta}\|_2\leq G ~~\text{and}~~\|\nabla^2 f_{\theta}\|_2\leq \Lambda.
\end{equation*}
We denote by $D:=\max(BG, 1/(1-\gamma))$.
\end{assumption}

\begin{assumption}[$\epsilon$-cover]\label{ass:NPG-COVER}
For the function class $\cG_{\cF}$, for any $\epsilon>0$, there exists an $\epsilon$-cover $\cC(\cG_{\cF},\epsilon)\subseteq\cG_{\cF}$ with size $|\cC(\cG_{\cF},\epsilon)|\leq \cN(\cG_{\cF}, \epsilon)$ such that for any $g\in\cG_{\cF}$, there exists $g'\in\cC(\cG_{\cF}, \epsilon)$ with $\|g-g'\|_\infty\leq \epsilon$.
\end{assumption}
We provide the sample complexity guarantee for \alg-NPG-SAMPLE as below.
\begin{theorem}[Sample Complexity of \algnpgsample]\label{thm:NPG-Sample}Let $\delta\in(0, 1)$ and $\varepsilon\in(0, 1/(1-\gamma))$. Suppose
Assumptions \ref{ass:eval_close_npg}, \ref{ass:NPG-REG}, and \ref{ass:NPG-COVER} hold. With proper hyperparameters, \alg-NPG-SAMPLE returns a policy $\pi$ satisfying \mbox{$V^\pi \geq V^{\pi^\star} - \varepsilon$} with probability at least $1-\delta$ after taking at most
\begin{align}
      \widetilde{O}\Big( \frac{D^6(D^2+\Lambda B^2)\cdot\big(\de(\cG_{\cF}, \beta)\big)^2\cdot \big(\log(\cN(\cG_\cF, \epsilon'))\big)^2}{\varepsilon^{8}(1-\gamma)^{8}}\Big)
\end{align}
samples, where $\beta=\varepsilon(1-\gamma)/2$ and $\epsilon'=poly(\varepsilon, \gamma, 1/D, 1/\de(\cG_\cF, \beta))$\footnote{The formal definition of $\epsilon'$ can be found in Theorem \ref{thm:NPG-SAMPLE-DETAIL}.}.
\end{theorem}

Notice that the differences between Theorems~\ref{thm:SPI-Sample} and~\ref{thm:NPG-Sample} only arise in the function class complexity terms and the regularity parameters, where the NPG version pays the complexity of the tangent class instead of the class $\cF$ as in the SPI case. NPG, however, offers algorithmic benefits as remarked before, and the result here extends to a more general form under a bounded transfer error condition that we present in Appendix \ref{app:NPG}. As with the algorithms, the theorems essentially coincide in the linear case. One interesting question for further investigation is the relationship between the eluder dimensions of the classes $\cF$ and $\cG_{\cF}$, which might inform statistical preferences between the two approaches.

\section{Experiment}\label{sec:exp}
We conduct experiments to testify the effectiveness of ENIAC. Specifically, we aim to show that
\begin{enumerate*}
\item ENIAC is competent to solve RL problem which requires exploration.
\item Compared with PC-PG which uses linear feature for bonus design, the idea of width in ENIAC performs better when using complex neural networks.
\end{enumerate*}
Our code is available at \href{https://github.com/FlorenceFeng/ENIAC}{https://github.com/FlorenceFeng/ENIAC}.

\subsection{Implementation of ENIAC}
We implement ENIAC using PPO \cite{schulman2017proximal} as the policy update routine and use fully-connected neural networks (FCNN) to parameterize actors and critics. At the beginning of each epoch $n\in[N]$, as in Algorithm \ref{alg:ENIAC}, we add a policy $\pi^n$ into the cover set, generate visitation samples following $\pi^n$, update the replay buffer to $\cZ^n$, and compute an approximate width function $w^n:\cS\times\cA\rightarrow\RR$. Recall that $w^n(s,a)$ is rigorously defined as: 
\begin{align}\label{eq:width-loss-ideal}
    \max_{f,f'\in\cF} ~f(s,a)-f'(s,a), \text{~subject to~} \|f-f'\|_{\cZ^n}\leq\epsilon.
\end{align}
To approximate this value in a more stable and efficient manner, we make several revisions to \eqref{eq:width-loss-ideal}: 
\begin{enumerate*}
    \item instead of training both $f$ and $f'$, we fix $f'$ and only train $f$;
    \item due to 1., we change the objective from $f-f'$ to $(f-f')^2$ for symmetry;
    \item instead of always retraining $f$ for every query point $(s,a)$, we gather a batch of query points $\cZ^n_Q$ and train in a finite-sum formulation.
\end{enumerate*}
Specifically, we initialize $f$ as a neural network with the same structure as the critic network (possibly different weights and biases) and initialize $f'$ as a copy of $f$. Then we fix $f'$ and train $f$ by maximizing the following loss:
\begin{align}\label{eq:width-loss-train}
\sum_{(s,a)\in\cZ_Q^n} \frac{\lambda\big(f(s,a)-f'(s,a)\big)^2}{|\cZ^n_Q|} - \sum_{(s',a')\in\cZ^n} \frac{\big(f(s',a')-f'(s',a')\big)^2}{|\cZ^n|}  -  \sum_{(s,a)\in\cZ_Q^n} \frac{\lambda_1\big(f(s,a)-f'(s,a)\big)}{|\cZ_Q^n|},\nonumber\\
\end{align}
where the last term is added to avoid a zero gradient (since $f$ and $f'$ are identical initially). We generate $\cZ_Q^n$ by using the current policy-cover as the initial distribution then rolling out with $\pi^n$. \eqref{eq:width-loss-train} can be roughly regarded as a Lagrangian form of \eqref{eq:width-loss-ideal} with regularization. The intuition is that we want the functions to be close on frequently visited area (the second term) and to be as far as possible on the query part (the first term). If a query point is away from the frequently visited region, then the constraint is loose and the difference between $f$ and $f'$ can be enlarged and a big bonus is granted; otherwise, the constraint becomes a dominant force and the width is fairly small. After training for several steps of stochastic gradient descent, we freeze both $f$ and $f'$ and return $|f(s,a)-f'(s,a)|$ as $w^n(s,a)$.
During the experiments, we set bonus as $0.5\cdot \frac{ w^n(s,a)}{\max_{\cZ^n_Q}w^n}$ without thresholding for the later actor-critic steps. A more detailed width training algorithm can be found in Appendix \ref{app:algo}.

We remark that in practice width training can be fairly flexible and customized for different environments. For example, one can design alternative loss functions as long as they follow the intuition of width; $f$ and $f'$ can be initialized with different weights and the loss function plays a pulling-in role instead of a stretching-out force as in our implementation; $\cZ_Q^n$ can be generated with various distributions as long as it has a relatively wide cover to ensure the quality of a batch-trained width. 




\subsection{Environment and Baselines}
We test on a continuous control task which requires exploration: continuous control MountainCar\footnote{\href{https://gym.openai.com/envs/MountainCarContinuous-v0/}{https://gym.openai.com/envs/MountainCarContinuous-v0/}} from OpenAI Gym~\cite{brockman2016openai}. This environment has a 2-dimensional continuous state space and a 1-dimensional continuous action space $[-1,1]$. The agent only receives a large reward $(+100)$ if it can reach the top of the hill and small negative rewards for any action. A locally optimal policy is to do nothing and avoid action costs. The length of horizon is 100 and $\gamma=0.99$.

We compare five algorithms: ENIAC, vanilla PPO, PPO-RND, \pcpg, and ZERO. All algorithms use PPO as their policy update routine and the same FCNN for actors and critics. The vanilla PPO has no bonus; PPO-RND uses RND bonus~\cite{burda2018exploration} throughout training; \pcpg iteratively constructs policy cover and uses linear features (kernel-based) to compute bonus as in the implementation of~\citet{agarwal2020pc}, which we follow here; ZERO uses policy cover as in \pcpg and the bonus is all-zero. For ENIAC, PC-PG, and ZERO, instead of adding bonuses to extrinsic rewards, we directly take the larger ones, i.e., the agent receives $\max(r, b)$ during exploration\footnote{This is simply for implementation convenience and does not change the algorithm. One can also adjust bonus as $\max(r,b)-r$.}. In ENIAC, we use uniform distribution to select policy from the cover set, i.e., $\rho^n_{\text{cov}}=\unif(d^{\pi^1}, \dots, d^{\pi^n})$ as in the main algorithm; PC-PG optimizes the selection distribution based on the policy coverage (see \cite{agarwal2020pc} for more details). We provide hyperparameters for all the methods in the Appendix \ref{app:algo}.

\begin{figure*}[!t]
    \centering
    \includegraphics[scale=0.38]{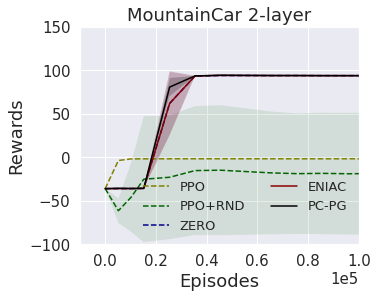}
    ~\includegraphics[scale=0.38]{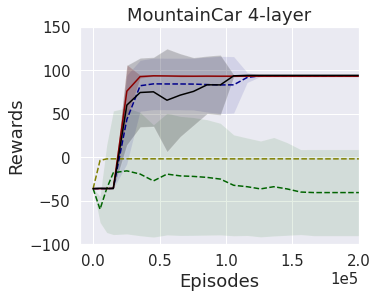}
    ~\includegraphics[scale=0.38]{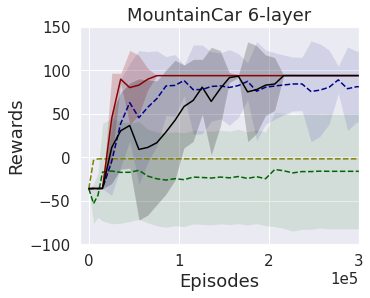}
    \vspace{-.1cm}
    \caption{Performance of different methods on MountainCar as we vary the netural network depth. The performances are evaluated over 10 random seeds where lines are means and shades represent standard deviations. We stop training once the policy can obtain rewards$>93$. 
    \vspace{-2mm}}
    \label{fig:rewards}
\end{figure*}

\begin{figure*}[!t]
    \centering 
     \includegraphics[scale=0.35]{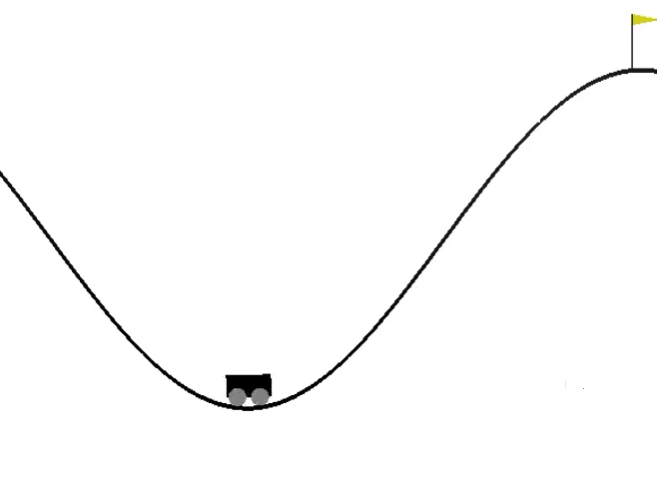}
     ~
    \includegraphics[scale=0.7]{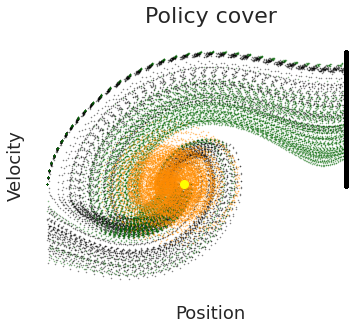}
    ~~
    \includegraphics[scale=0.7]{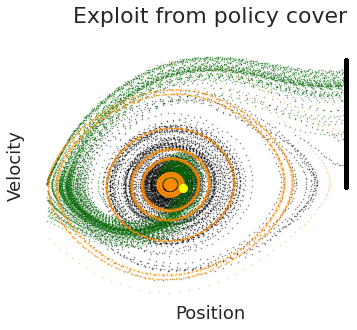}
    \caption{The MountainCar environment (left). Trajectories of exploration (middle) and exploitation (right) policies of \alg, with colors denoting different epochs: orange for the first policy in the cover set, black for the second, and green for the third. Agent starts from the centric area (near the yellow circle) and the black vertical line on the right represents goal positions. }
    \label{fig:visitation}
\end{figure*}

\begin{figure*}[!t]
    \centering
    \includegraphics[scale=0.3]{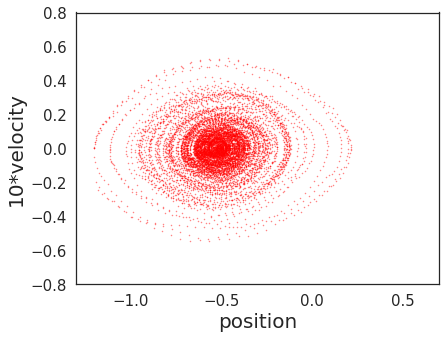}
    ~\includegraphics[scale=0.3]{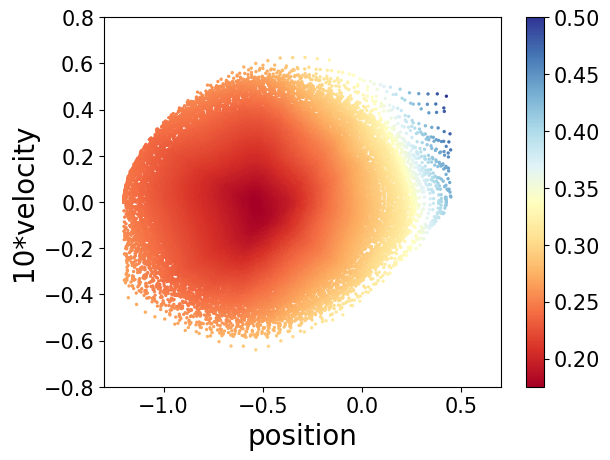}
    ~\includegraphics[scale=0.3]{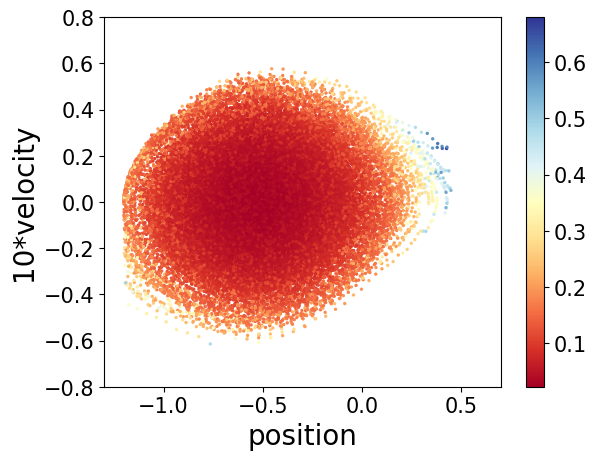}
    \caption{Bonus function comparison. [Left]: The trajectories of a chosen policy . [Middle]: the bonus function built by ENIAC upon the policy. [Right]: The bonus built by PC-PG upon the policy. See text for details.}
    \label{fig:bonus}
\end{figure*}

\subsection{Results}
We evaluate all the methods for varying depths of the critic network: 2-layer stands for (64, 64) hidden units, 4-layer for (64, 128, 128, 64), and 6-layer for (64, 64, 128, 128, 64, 64). Layers are connected with ReLU non-linearities for all networks. In Figure~\ref{fig:rewards}, we see that \alg robustly achieves high performance consistently in all cases. Both \pcpg and ZERO perform well for depth 2, but as we increase the depth, the heuristic kernel-based bonus and the 0-offset bonus do not provide a good representation of the critic's uncertainty and its learning gets increasingly slower and unreliable. PPO and PPO-RND perform poorly, consistent with the results of~\citet{agarwal2020pc}. One can also regard the excess layers as masks on the true states and turn them into high-dimensional observations. When observations become increasingly complicated, more non-linearity is required for information processing and ENIAC is a more appealing choice.

We visualize \alg's policies in Figure~\ref{fig:visitation}, where we plot the state visitations of the exploration policies from the cover, as well as the exploitation policies trained using the cover with just the external reward, for varying number of epochs. We see that \alg quickly attains exploration in the vicinity of the optimal state, allowing the exploitation policy to become optimal. Since the bonus in our experiments is smaller than the maximum reward, the effect of the bonus dissipates once we reach the optimal state, even for the exploration policies. We also visualize typical landscapes of bonus functions in ENIAC and \pcpg in Figure \ref{fig:bonus}. Both bonuses grant small values on frequently visited area and large values on scarsely visited part. But the bonus in ENIAC changes in a smoother way than the one in \pcpg. This might inspire future study on the shaping of bonuses.


The results testify the competence of ENIAC on the exploration problem. Especially, compared with PC-PG, the usage of width is more suitable for complex function approximation.


\section{Conclusion}
In this paper, we present the first set of policy-based techniques for RL with non-linear function approximation. Our methods provide interesting tradeoffs between sample and computational complexities, while also inspire an extremely practical implementation. Empirically, our results demonstrate the benefit of correctly reasoning about the learner's uncertainty under a non-linear function class, while prior heuristics based on linear function approximation fail to robustly work as we vary the function class. Overall, our results open several interesting avenues of investigation for both theoretical and empirical progress. In theory, it is quite likely that our sample complexity results have scope for a significant improvement. A key challenge here is to enable better sample reuse, typically done with bootstrapping techniques for off-policy learning, while preserving the robustness to model misspecification that our theory exhibits. Empirically, it would be worthwhile to scale these methods to complex state and action spaces such as image-based inputs, and evaluate them on more challenging exploration tasks with a longer effective horizon.

\bibliographystyle{plainnat}
\bibliography{references}
\newpage
{
\hypersetup{linkcolor=black}
\renewcommand{\baselinestretch}{1.0}\normalsize
\tableofcontents
\renewcommand{\baselinestretch}{1.0}\normalsize
}

\newpage
\appendix
\cut{
\section{\textcolor{blue}{Corrections to the main body}}

In this section, we detail minor changes to our technical assumptions and their effects on our sample complexity results. We begin with the SPI algorithms.

\subsection{Updates for SPI results}

Assumption~\ref{ass:eval_close} needs to be strengthened to hold for all policies as described below.
\begin{assumption}[Correction to Assumption~\ref{ass:eval_close}]
For all $\pi\in \{\cS\to\Delta(\cA)\}$ and $g~:~\cS\times\cA\to[0,\tfrac{2}{(1-\gamma)^2}]$, we have $\cT^{\pi} g \in \cF$.
\label{ass:eval_close_fix}
\end{assumption}

The reason for this stronger assumption is that SPI constructs policies which are not greedy policies $\pi_f$, but rather take the form $\pi_t\propto \exp(\sum_{i=1}^t f_i)$ on the set $\cK^n$ at epoch $n$, and uniform otherwise. Since the set $\cK^n$ is not easily described in any parametric or functional form, we make the assumption for all randomized policies. Note that this assumption holds, for instance, when $\cF$ consists of linear functions in the features defining a linear MDP~\citep{yang2019reinforcement,jin2019provably}.

The second change is that we need an additional complexity assumption on $\cF$ to enable uniform convergence like in~\citet{wang2020reinforcement}. This assumption is identical to the one made in their earlier work, and a fairly standard one for learning over a function class $\cF$.

\begin{assumption}[$\epsilon$-cover]\label{ass:compact}
Given a function class $\cF$, for any $\epsilon>0$, there exists an $\epsilon$-cover $\cC(\cF,\epsilon)\subseteq\cF$ with size $|\cC(\cF,\epsilon)|\leq \cN(\cF, \epsilon)$. such that for any $f\in\cF$, there exists $f'\in\cC(\cF, \epsilon)$ with $\|f-f'\|_\infty\leq \epsilon$.
\end{assumption}

\begin{theorem}[Correction of Theorem~\ref{thm:SPI-Sample}]
Let $\delta\in(0, 1)$ and $\varepsilon\in(0, 1/(1-\gamma))$. Suppose
Assumptions \ref{ass:eval_close_fix} and~\ref{ass:compact} holds and $\|f\|_{\infty}\leq W$ for every epoch $n\in[N]$ and every $f\in\cF$. With proper hyperparameters, \alg-SPI returns a policy $\pi$ satisfying \mbox{$V^\pi \geq V^{\pi^\star} - \varepsilon$}, after taking at most
\begin{align}
  \widetilde{O}\Big(\frac{W^8\big(\de(\cF, \beta)\big)^3\cdot \log(\cN(\cF, \epsilon'))}{\varepsilon^{11}(1-\gamma)^{15}}\Big),
\end{align}
with probability at least $1-\delta$,
where 
$\widetilde{O}$ hides polylog factors in $|\cA|, W, 1/\varepsilon$  and $1/(1-\gamma)$.
\end{theorem}

As clearly seen, the main difference in sample complexity is the appearance of the log-covering number. In the linear setting, this only introduces an extra factor of $d$ for features $\phi\in \mathbb{R}^d$.

\subsection{Updates for NPG results}
Similarly, Assumption~\ref{ass:eval_close_npg} needs to be strengthened to hold for all policies as described below.
\begin{assumption}[Correction to Assumption~\ref{ass:eval_close}]
Let $\pi_f(a|s)\propto\exp(f(s,a)), f\in\cF$. For any measurable set $\cK\in\cS\times\cA$ and $g~:~\cS\times\cA\to[0,\frac{4}{(1-\gamma)^2}]$, we have $\cT^{\pi_{f,\cK}} g-\mathbb{E}_{\pi_{f,\cK}}[\cT^{\pi_{f,\cK}}g] \in \cF$, where
\begin{align}
    \pi_{f,\cK}(\cdot|s) &= \pi_f(\cdot|s), \quad \text{if for all } a\in\cA, (s,a)\in\cK\\
    \pi_{f,\cK}(\cdot|s) &= \unif(\{a| (s,a)\notin \cK\}),\quad o.w.
\end{align}
\label{ass:eval_close_npg_fix}
\end{assumption}
The reason for this stronger assumption is that NPG constructs policies which are not greedy policies $\pi_{f}$, but rather $\pi_t\propto\exp(\pi_{f}\cdot\indict\{\cK^n\})$ and the latter may not represented as a parameterized policy.

As in SPI, we also need an additional complexity assumption on $\cG_{\cF}$

\begin{assumption}[$\epsilon$-cover]\label{ass:NPG-COVER}
Given a function class $\cG_{\cF}$, for any $\epsilon>0$, there exists an $\epsilon$-cover $\cC(\cG_{\cF},\epsilon)\subseteq\cG_{\cF}$ with size $|\cC(\cG_{\cF},\epsilon)|\leq \cN(\cG_{\cF}, \epsilon)$ such that for any $g\in\cG_{\cF}$, there exists $g'\in\cC(\cG_{\cF}, \epsilon)$ with $\|g-g'\|_\infty\leq \epsilon$.
\end{assumption}

\begin{theorem}[Correction of Theorem~\ref{thm:NPG-Sample}]
Let $\delta\in(0, 1)$ and $\varepsilon\in(0, 1/(1-\gamma))$. Suppose
Assumptions \ref{ass:eval_close_npg_fix}, \ref{ass:NPG-COVER}, and \ref{ass:NPG-REG} hold. With proper hyperparameters, \alg-NPG returns a policy $\pi$ satisfying \mbox{$V^\pi \geq V^{\pi^\star} - \varepsilon$}, after taking at most
\begin{align}
    \widetilde{O}\Big(\frac{B^8G^6(G^2+\Lambda)\big(\de(\cG_{\cF}, \beta)\big)^3\cdot \log(\cN(\cG_{\cF},  \epsilon')/\delta)}{\epsilon^{11}(1-\gamma)^{15}}\Big)
\end{align}
samples with probability at least $1-\delta$,
where 
$\widetilde{O}$ hides the polylog factor in $|\cA|, B, G, \Lambda, 1/\varepsilon$, and $1/(1-\gamma)$.
\end{theorem}

As clearly seen, the main difference in sample complexity is the appearance of the log-covering number. In the linear setting, this only introduces an extra factor of $d$ for features $\phi\in \mathbb{R}^d$.
}
\section{Omitted pseudocodes from main text}

We give the pseudocodes for value estimators and visitation distribution sampler in Algorithms~\ref{alg:estimator} and~\ref{alg:sampler} respectively. Combining them, we are able to generate samples for critic fit.

\begin{algorithm}[t]
  \caption{Value Estimators}
  \label{alg:estimator}
  \begin{algorithmic}[1]
   \State \textbf{Routine:} $V^{\pi}$-\texttt{ESTIMATOR}
  \State \hskip1em \textbf{Input:} starting state $s$.
  \State \hskip1em Execute $\pi$ from $s$; at any step $t$ with $(s_t,a_t)$, terminate with probability $1-\gamma$.
  \State \hskip1em \textbf{Return:} $\hat{V}^{\pi}(s)=\sum_{i=0}^t r(s_i,a_i)$, where $s_0=s$.\\
  \vspace{5pt}
  \State \textbf{Routine:} $Q^{\pi}$-\texttt{ESTIMATOR}
  \State \hskip1em \textbf{Input:} starting state-action $(s,a)$.
  \State \hskip1em Execute $\pi$ from $(s,a)$; at any step $t$ with $(s_t,a_t)$, terminate with probability $1-\gamma$.
  \State \hskip1em \textbf{Return:} $\hat{Q}^{\pi}(s,a)=\sum_{i=0}^t r(s_i,a_i)$, where $(s_0,a_0)=(s,a)$.
  \cut{\State \textbf{Routine:} $Q^{\pi}$-\texttt{ESTIMATOR for Stop Policies}
  \State \hskip1em \textbf{Input:} starting state-action $(s,a)$, a stop policy $\pi$, $r(s,a)$.
  \State \hskip1em Execute $\pi$ from $(s,a)$;
  \State \hskip1em \textbf{Return:} $\hat{Q}^{\pi}(s,a)=\sum_{i=0}^t \gamma^i \cdot r(s_i,a_i)$, where $(s_0,a_0)=(s,a)$.
  \\}
\cut{  \State \textbf{Routine:} $Q^{\pi}$-\texttt{ESTIMATOR for a terminating policy $\tpi$}
  \State \hskip1em \textbf{Input:} starting state-action $(s,a)$, $\pi, r(s,a)$.
  \State \hskip1em Execute $\tpi$ from $(s,a)$.
  \State \hskip1em \textbf{Return:} $\hat{Q}^{\pi}_{\texttt{ONE}}(s,a)=\sum_{i=0}^{t} \gamma^t \cdot r(s_i,a_i)$, where $(s_0,a_0)=(s,a)$.}
  \end{algorithmic}
\end{algorithm}

\begin{algorithm}[t]
  \caption{$d^{\pi}$ Sampler}
  \label{alg:sampler}
  \begin{algorithmic}[1]
  \State \textbf{Routine:} $d^{\pi}_{\nu}$-\texttt{SAMPLER}
  \State \hskip1em \textbf{Input:} $\nu\in\Delta(\cS\times\cA), \pi$.
  \State \hskip1em Sample $s_0,a_0\sim \nu$;
  \State \hskip1em Execute $\pi$ from $s_0,a_0$; at any step $t$ with $(s_t,a_t)$, terminate with probability $1-\gamma$.
  \State \hskip1em \textbf{Return:} $(s_t,a_t)$.
  \cut{\State \textbf{Routine:} $d^{\pi}_{\nu}$-\texttt{ONE-STEP-FORWARD-SAMPLER}
  \State \hskip1em \textbf{Input:} $\nu\in\Delta(\cS\times\cA), \pi$.
  \State \hskip1em Run $d^{\pi}_{\nu}$-\texttt{SAMPLER} and assume the termination happens at $(s_t,a_t,s_{t+1})$
  \State \hskip1em Keep executing one more an action $a_{t+1}\sim \unif(\cA)$ from $s_{t+1}$ and reach a new state $s'$.
  \State \hskip1em \textbf{Return:} $(s_{t+1}, a_{t+1}, s')$.}
  \end{algorithmic}
\end{algorithm}

\section{Proof Setup}
\subsection{Definition and Notation}\label{sec:proofdef}
We denote by $\cM$ the original MDP and $\tilde{\pi}$ an arbitrary fixed comparator policy (e.g., an optimal policy). Our target is to show that after $N$ epochs, ENIAC is able to output a policy whose value is larger than $V^{\tilde{\pi}}$ minus some problem-dependent constant. First we describe the construction of some auxiliary MDPs, which is conceptually similar to~\citet{agarwal2020pc}, modulo the difference in the bonus functions.

For each epoch $n\in[N]$, we  consider three MDPs: the original MDP $\cM$, the bonus-added MDP $\cM_{b^n}:=(\cS, \cA, P, r+b^n, \gamma)$, and an auxiliary MDP $\cM^n$. $\cM^n$ is defined as $(\cS,\cA\cup\{a^\dagger\}, P^n, r^n, \gamma)$, where $a^\dagger$ is an extra action which is only available for $s\notin \cK^n$ (recall that $s\in\cK^n$ if and only if $b^n(s,a)\equiv 0$ for all $a\in\cA$). For all $(s,a)\in\cS\times\cA$,
\begin{align}
   P^n(\cdot|s,a) = P(\cdot|s,a), ~~~ r^n(s,a) = r(s,a) + b^n(s,a).
\end{align}
For $ s\notin\cK^n$,
\begin{align}
  P^n(s|s,a^\dagger)=1, ~~~  r^n(s,a^{\dagger})=1.
\end{align}
Basically, $a^\dagger$ allows the agent to stay in a state $s\notin\cK^n$ while accumulating maximum instant rewards.

Given $\cM^n$, we further define $\tilde{\pi}^n$ such that $\tilde{\pi}^n(\cdot|s) = \tilde{\pi}(\cdot|s)$ for $s\in\cK^n$ and $\tilde{\pi}^n(a^{\dagger}|s)=1$ for $s\notin\cK^n$. We denote by $\tilde{d}_{\cM^n}$ the state-action distribution induced by $\tilde{\pi}^n$ on $\cM^n$ and $d^{\tilde{\pi}}$ the state-action distribution induced by $\tilde{\pi}$ on $\cM$.

\paragraph{Additional Notations}
Given a policy $\pi$, we denote by $V^\pi_{b^n}, Q^\pi_{b^n}$, and $A^\pi_{b^n}$ the state-value, $Q$-value, and advantage function of $\pi$ on $\cM_{b^n}$ and $V^\pi_{\cM^n}, Q^\pi_{\cM^n}$, and $A^\pi_{\cM^n}$ for the counterparts on $\cM^n$. For the policy $\pi^n_t$, i.e., the policy at the $t_{\text{th}}$ iteration in the $n_{\text{th}}$ epoch of ENIAC, we further simplify the notation as $V^t_{b^n}, Q^t_{b^n}$, and $A^t_{b^n}$ and also $V^t_{\cM^n}, Q^t_{\cM^n}$, and $A^t_{\cM^n}$.

\begin{remark}\label{rmk:equiv}
Note that only $\tilde{\pi}^n$ can take the action $a^{\dagger}$ for $s\notin\cK^n$. All policies $\{\pi^n_t\}$ is not aware of $a^{\dagger}$ and therefore, $V^t_{b^n}=V^t_{\cM^n}$, $Q^t_{b^n}=Q^t_{\cM^n}$, and $A^t_{b^n}=A^t_{\cM^n}$.
\end{remark}

Based on the above definitions, we directly have the following two lemmas.

\begin{lemma}\label{lemma:dbound}
Consider any state $s\in\cK^n$, we have:
\begin{align}
    \tilde{d}_{\cM^n}(s,a)\leq d^{\tilde{\pi}}(s,a), ~~~\forall a\in\cA.
\end{align}
\end{lemma}
\begin{proof}
The proof follows that of Lemma B.1. in \cite{agarwal2020pc}. We present below for the readers' convenience.

We prove by induction over the time steps along the horizon. Recall $\tilde{d}_{\cM^n}$ is the state-action distribution of $\tilde{\pi}^n$ over $\cM^n$ and $d^{\tilde{\pi}}$ is the state-action distribution of $\tilde{\pi}$ on both $\cM_{b^n}$ and $\cM$ as they share the same dynamics. We use another subscript $h$ to indicate the step index, e.g., $\tilde{d}_{\cM^n,h}$ is the state-action distribution at the $h_{\text{th}}$ step following $\tilde{\pi}^n$ on $\cM^n$.

Starting at $h=0$, if $s_0\in\cK^n$, then $\tilde{\pi}^n(\cdot|s_0)=\tilde{\pi}(\cdot|s_0)$ and we can easily get:
\begin{align}
    \tilde{d}_{\cM^n,0}(s_0,a)= d^{\tilde{\pi}}_0(s_0,a), ~~~\forall a\in\cA.
\end{align}
Now we assume that at step $h$, for all $s\in\cK^n$, it holds that
 \begin{align}
     \tilde{d}_{\cM^n,h}(s,a)\leq d^{\tilde{\pi}}_h(s,a), ~\forall a\in\cA.
 \end{align}
 Then, for step $h+1$, by definition we have that for $s\in\cK^n$
 \begin{align}
     \tilde{d}_{\cM^n, h+1}(s) &= \sum_{s',a'} \tilde{d}_{\cM^n,h}(s',a') P_{\cM^n}(s|s',a')\\
     &=\sum_{s',a'}\indict\{s'\in\cK^n\} \tilde{d}_{\cM^n,h}(s',a') P_{\cM^n}(s|s',a')\\
     &=\sum_{s',a'}\indict\{s'\in\cK^n\} \tilde{d}_{\cM^n,h}(s',a') P(s|s',a'),
 \end{align}
 where the second line is due to that if $s'\notin \cK^n$, $\tilde{\pi}$ will deterministically pick $a^\dagger$ and $P_{\cM^n}(s|s',a^\dagger)=0$. On the other hand, for $d^{\tilde{\pi}}_{h+1}(s,a)$, it holds that for $s\in\cK^n$,
 \begin{align}
     d^{\tilde{\pi}}_{h+1}(s) &= \sum_{s',a'}d^{\tilde{\pi}}_h(s',a')P(s|s',a') \\
     &= \sum_{s',a'}\indict\{s'\in\cK^n\} d^{\tilde{\pi}}_{h}(s',a') P(s|s',a') + \sum_{s',a'}\indict\{s'\notin\cK^n\} d^{\tilde{\pi}}_{h}(s',a')P(s|s',a')\\
     &\geq \sum_{s',a'}\indict\{s'\in\cK^n\} d^{\tilde{\pi}}_h(s',a')P(s|s',a')\\
     &\geq \sum_{s',a'}\indict\{s'\in\cK^n\}\tilde{d}_{\cM^n,h}(s',a')P(s|s',a')=\tilde{d}_{\cM^n,h+1}(s).
 \end{align}
 Using the fact that $\tilde{\pi}^n(\cdot|s) = \tilde{\pi}(\cdot|s)$ for $s\in\cK^n$, we conclude that the inductive hypothesis holds at $h+1$ as well. Using the definition of the average state-action distribution, we conclude the proof.
\end{proof}

\begin{lemma}\label{lemma:Vpibound}
For any epoch $n\in[N]$, we have
\begin{align}
    &V^{\tilde{\pi}^n}_{\cM^n} \geq V^{\tilde{\pi}}_{\cM}.
\end{align}
\end{lemma}
\begin{proof}
The result is straightforward since if following $\tilde{\pi}^n$ we run into some $s\notin\cK^n$, then by definition, $\tilde{\pi}^n$ is able to collect maximum instant rewards for all steps later.
\end{proof}

\subsection{Proof Sketch}\label{sec:proofidea}
We intend to compare the values of the output policy $\pi^{N}_{\text{ave}}:=\unif(\pi^{2}, \pi^3, \dots, \pi^{N+1})$ and the comparator $\tilde{\pi}$. To achieve this, we use two intermediate quantities
$V^{\pi^{n+1}}_{b^n}$ and $V^{\tilde{\pi}^n}_{\cM^{n}}$ and build the following inequalities as bridges:
\begin{align}\label{eq:idea}
    V^{\pi^{N}_{\text{ave}}}=\frac{1}{N}\sum_{n=1}^N V^{\pi^{n+1}} \geq \frac{1}{N}\sum_{n=1}^N V^{\pi^{n+1}}_{b^{n}} - A,  \quad V^{\pi^{n+1}}_{b^{n}} =V^{\pi^{n+1}}_{\cM^n}&\geq V^{\tilde{\pi}^{n}}_{\cM^{n}} - B, \quad V^{\tilde{\pi}^{n}}_{\cM^{n}}\geq V^{\tilde{\pi}},
\end{align}
where $A$ and $B$ are two terms to be specified. If the above relations all hold, the desired result is natually induced. For these inequalities, we observe that
\begin{enumerate}
    \item The leftmost inequality is about the value differences of a sequence of policies $(\pi^{2}, \pi^3, \dots, \pi^{N+1})$ on two different reward functions (with or without the bonus). Thus, it is bounded by the cumulative bonus, or equivalently, the \emph{expected bonus} over the state-action measure induced by these policies, which we use the eluder dimension of the approximation function class to bound. We present this result for SPI-Sample, SPI-Compute, and NPG-Sample in Lemma \ref{lemma:SPI-I-Bonus}, \ref{lemma:SPI-II-Bonus}, and \ref{lemma:NPG-I-Bonus}, respectively.
    \item The rightmost inequality is proved in Lemma \ref{lemma:Vpibound}.
    \item To show the middle inequality, we analyze the convergence of actor-critic updates, leveraging properties of the multiplicative weight updates for a regret bound following the analysis of~\citet{agarwal2020theory}.
\end{enumerate}
In the sequel, we present sample complexity analysis for ENIAC-SPI-SAMPLE, ENIAC-SPI-COMPUTE, and ENIAC-NPG-SAMPLE. ENIAC-NPG-COMPUTE can be easily adapted with minor changes of the assumptions. In particular, we provide general results considering model misspecification and the theorems in the main body fall as special cases under Assumption \ref{ass:eval_close} or \ref{ass:eval_close_npg}.

\section{Analysis of ENIAC-SPI}\label{app:SPI}
In this section, we provide analysis for ENIAC-SPI-SAMPLE and ENIAC-SPI-COMPUTE. 
We start with stating the assumptions which quantifies model misspecification.
\begin{assumption}[Bounded Transfer Error]\label{ass:SPI-boundedtransfer}
Given a target function $g:\cS\times\cA\rightarrow\RR$, we define the critic loss function $L(f;d,g)$ with $d\in \Delta(\cS\times\cA)$ as:
\begin{align}
    L(f;d,g):= \mathbb{E}_{(s,a)\sim d} \left[\big(f(s,a)-g(s,a)\big)^2\right].
\end{align}
For the fixed comparator policy $\tilde{\pi}$ (defined at the beginning of Section \ref{sec:proofdef}), we define $\tilde{d}(s,a):=d^{\tilde{\pi}}_{s_0}(s)\circ \unif(\cA)$. In ENIAC-SPI (both sample and compute versions), for every epoch $n\in[N]$ and every iteration $t$ inside epoch $n$, we assume that
\cut{\begin{align}
    \tilde{\cF}_t^n\in\argmin_{f\in\cF} L(f;\rho^n_{\text{cov}}, Q^{t}_{b^n}-b^n)
\end{align}}
\begin{align}\label{eq:Lbias}
    \inf_{f\in\cF_t^n} L(f; \tilde{d}, Q^{t}_{b^n}-b^n)\leq \epsilon_{\text{bias}},
\end{align}
where $\cF_t^n:=\argmin_{f\in\cF} L(f;\rho^n_{\text{cov}}, Q^{t}_{b^n}-b^n)$ and $\epsilon_{\text{bias}}\geq 0$ is some problem-dependent constant.
\end{assumption}
$\epsilon_{\text{bias}}$ measures both approximation error and distribution shift error. In later proof, we select a particular function in $\tilde{f}^n_t\in\cF^n_t$ such that
\begin{align}\label{eq:tildef}
 L(\tilde{f}^n_t; \tilde{d}, Q^{t}_{b^n}-b^n)\leq 2\epsilon_{\text{bias}}.
\end{align}
We establish complexity results by comparing the empirical minimizer $f^n_t$ of \eqref{eqn:spi-critic-fit} with this optimal fitter $\tilde{f}^n_t$.


\begin{assumption}\label{ass:SPI-quadra}
For the same loss $L$ as defined in Assumption \ref{ass:SPI-boundedtransfer} and the fitter $\tilde{f}^n_t$, we assume that there exists some $C\geq 1$ and $\epsilon_0\geq 0$ such that for any $f\in\cF$,
\cut{\begin{align}
    &\mathbb{E}_{(s,a)\sim\rho^n_{\text{cov}}} \left[\big(f^n_t(s,a)-\tilde{f}^n_t(s,a)\big)^2\right]\leq C\cdot\Big(L(f^n_t; \rho^n_{\text{cov}}, Q^{t}_{b^n}-b^n) - L(\tilde{f}^n_t; \rho^n_{\text{cov}}, Q^{t}_{b^n}-b^n)\Big) + \epsilon_0,
\end{align}}
\begin{align}
    &\mathbb{E}_{(s,a)\sim\rho^n_{\text{cov}}} \left[\big(f(s,a)-\tilde{f}^n_t(s,a)\big)^2\right]\leq C\cdot\Big(L(f; \rho^n_{\text{cov}}, Q^{t}_{b^n}-b^n) - L(\tilde{f}^n_t; \rho^n_{\text{cov}}, Q^{t}_{b^n}-b^n)\Big) + \epsilon_0
\end{align}
for $n\in[N]$ and $0\leq t\leq T-1$.
\end{assumption}

\begin{remark}\label{rmk:spi}
Under Assumption \ref{ass:eval_close}, $Q^t_{b^n}-b^n = \mathbb{E}^{\pi^n_t}[r(s,a)+\gamma Q^t_{b^n}(s',a')]\in\cF$. Thus, $\epsilon_{\text{bias}}$ can take value 0 and $\tilde{f}^n_t=Q^t_{b^n}-b^n$. Further in Assumption \ref{ass:SPI-quadra}, we have
\begin{align}
    &\mathbb{E}_{(s,a)\sim\rho^n_{\text{cov}}} \left[\big(f(s,a)-\tilde{f}^n_t(s,a)\big)^2\right] = L(f; \rho^n_{\text{cov}}, Q^{t}_{b^n}-b^n).
\end{align}
Thus, $C$ can take value 1 and $\epsilon_0=0$. If $Q^t_{b^n}-b^n$ is not realizable in $\cF$, $\epsilon_{\text{bias}}$ and $\epsilon_0$ could be strictly positive. Hence, the above two assumptions are generalized version of the closedness condition considering model misspecification.
\end{remark}



\subsection{Sample Complexity of ENIAC-SPI-SAMPLE}\label{sec:SPI-SAMPLE-Proof}
We follow the proof steps in Section \ref{sec:proofidea} and first establish a bonus bound.
\begin{lemma}[SPI-SAMPLE: The Bound of Bonus]\label{lemma:SPI-I-Bonus}
With probability at least $1-N\delta$, it holds that
\begin{align}
    \sum_{n=1}^N \Big(V^{{\pi}^{n+1}}_{b^n} - V^{\pi^{n+1}}\Big) \leq  \frac{2\epsilon^2+8KW^2+\beta^2}{(1-\gamma)\beta^2K}\cdot\de(\cF, \beta) + \frac{N}{1-\gamma}\sqrt{\frac{\log(2/\delta)}{2K}}.
\end{align}
\end{lemma}
\begin{proof}
\begin{align}
    \sum_{n=1}^N \big(V^{\pi^{n+1}}_{b^n} - V^{\pi^{n+1}}\big)&\leq \sum_{n=1}^N ~\mathbb{E}_{(s,a)\sim d^{n+1}}\indict\{(s,a)\notin\cK^n\}/(1-\gamma)\\
    &=\sum_{n=1}^N ~\mathbb{E}_{(s,a)\sim d^{n+1}}\indict \{w(\tilde{\cF}^n,s,a)\geq \beta\}/(1-\gamma),
\end{align}
where $d^{n+1}$ denotes the state-action distribution induced by $\pi^{n+1}$ on $\cM$.
We denote by $\cD^{n}$ the sampled dataset $\{(s_i,a_i)\}_{i=1}^K\sim d^{n}$ at the beginning of epoch $n$. Then $\cZ^{n} = \cZ^{n-1}\cup \cD^n$. By Hoeffding's inequality, with probability at least $1-\delta$,
\begin{align}
\mathbb{E}_{(s,a)\sim d^{n+1}}\indict \{w(\tilde{\cF}^n,s,a)\geq\beta\} \leq \frac{1}{K}\sum_{(s,a)\in \cD^{n+1}} \indict \{w(\tilde{\cF}^n,s,a)\geq\beta\} + \sqrt{\frac{\log(2/\delta)}{2K}}.
\end{align}
Taking the union bound, with probability at least $1-N\delta$, we have
\begin{align}
    \sum_{n=1}^N V^{\pi^{n+1}}_{b^{n}}-V^{\pi^{n+1}} \leq \frac{1}{K(1-\gamma)}\sum_{n=1}^N\sum_{(s,a)\in \cD^{n+1}}\indict \{w(\tilde{\cF}^n,s,a)\geq \beta\} + \frac{N}{1-\gamma}\sqrt{\frac{\log(2/\delta)}{2K}}.\label{eq:boundbonus}
\end{align}
Next we bound the first term in Equation \eqref{eq:boundbonus} following a similar process as in \cite[Proposition 3]{russo2013eluder}.
We simplify $w(\tilde{ \cF}^n,\cdot,\cdot)$ as $w^n(\cdot,\cdot)$ and label all samples in $\cZ^n$ in lexical order, e.g., $(s^{n+1}_i,a^{n+1}_i)$ denotes the $i$th sample in $\cD^{n+1}$. For every $(s_i^{n+1}, a_i^{n+1})$, we define a sequence $S^{n+1}_{i-1}$ which contains all samples generated before $(s_i^{n+1}, a_i^{n+1})$, i.e.,
\begin{align}\label{eq:seq}
 S^{n+1}_{i-1}:= \big((s_1^1,a_1^1), \dots, (s^1_K, a^1_K), (s^2_1,a^2_1),\cdots (s_K^{n},a_K^{n}), (s^{n+1}_1, a^{n+1}_1),\dots,(s^{n+1}_{i-1}, a^{n+1}_{i-1})\big)
\end{align}
Next we show that,
\begin{align}\label{eq:wbound}
\sum_{n=1}^N\sum_{(s,a)\in\cD^{n+1}}\indict\{w^n(s, a)\geq \beta\}\leq \Big(2\epsilon^2/\beta^2 + 8W^2K/\beta^2+1\Big)\cdot\de(\cF, \beta).
\end{align}
For $n\leq N$, if $w^n(s^{n+1}_i,a^{n+1}_i) > \beta$ then $(s^{n+1}_i,a^{n+1}_i)$ is $\beta$-dependent with respect to $\cF$ on fewer than $8(\epsilon)^2/\beta^2 + 32W^2K/\beta^2$ disjoint subsequences of $S^{n+1}_{i-1}$. To see this, note that if $w^n(s^{n+1}_i, a^{n+1}_i)> \beta$, there exists $\bar{f},\underline{f}\in\cF$ such that $\bar{f}-\underline{f}\in \tilde{\cF}^n$ and $\bar{f}(s^{n+1}_i, a^{n+1}_i)-\underline{f}(s^{n+1}_i, a^{n+1}_i)\geq \beta$. By definition, if $(s^{n+1}_i, a^{n+1}_i)$ is $\beta$-dependent on a subsequence $\big((s_{t_1},a_{t_1}), \dots, (s_{t_k}, a_{t_k})\big)$ of $S^{n+1}_{i-1}$, then $\sum_{j=1}^k(\bar{f}\big(s_{t_j},a_{t_j}) - \underline{f}(s_{t_j},a_{t_j})\big)^2\geq \beta^2$. It follows that, if $(s^{n+1}_i,a^{n+1}_i)$ is $\beta$-dependent on $L$ disjoint subsequences of $S^{n+1}_{i-1}$ then $\|\bar{f}-\underline{f}\|^2_{S^{n+1}_{i-1}}\geq L\beta^2$, where we recall our notation $\|f\|_S = \sqrt{\sum_{x\in S} f(x)^2}$. By the definition of $\tilde{\cF}^n$ and $S^{n+1}_{i-1} = \cZ^{n}\cup \{(s^{n+1}_j, a^{n+1}_j)\}_{j=1}^{i-1}$, we have
\begin{align}
   \|\bar{f}-\underline{f}\|_{S^{n+1}_{i-1}}\leq\|\bar{f}-\underline{f}\|_{\cZ^{n}} +\|\bar{f}-\underline{f}\|_{\{(s^{n+1}_j, a^{n+1}_j)\}_{j=1}^{i-1}} \leq\epsilon + 2W\sqrt{i-1}\leq \epsilon+ 2W\sqrt{K},
\end{align}
where $W$ is an upper bound of $\|f\|_{\infty}$. Hence, $L< 2\epsilon^2/\beta^2 + 8W^2K/\beta^2$.

Next, we show that in any state-action sequence $((s_1,a_1), \dots, (s_{\tau}, a_{\tau}))$, there is some $j \leq \tau$ such that the element $(s_j,a_j)$ is $\beta$-dependent with respect to $\cF$ on at least $\tau/d-1$ disjoint subsequences of the subset $((s_1,a_1),\dots, (s_{j-1},a_{j-1}))$, where $d:=\de(\cF, \beta)$. Here we assume that $\tau\ge d$ since otherwise the claim is trivially true. To see this, for an integer $L$ safistying $Ld + 1\leq \tau \leq (L+1)\cdot d$, we will construct $L$ disjoint subsequences $S_1, \dots, S_L$ \emph{one element at a time}. First, for each $i\in[L]$ add $(s_i,a_i)$ to the subsequence $S_i$. Now, if $(s_{L+1}, a_{L+1})$ is $\beta$-dependent on all subsequences $S_1, \dots, S_L$, our claim is established. Otherwise, select a subsequence $S_i$ such that $(s_{L+1}, a_{L+1})$ is $\beta$-independent of it and append $(s_{L+1}, a_{L+1})$ to $S_i$. Repeat this process for elements with indices $j>L+1$ until $(s_j, a_j)$ is $\beta$-dependent on all subsequences or $j=\tau$. In the latter scenario, since $\tau-1$ elements have already been put in subsequences, we have that $\sum|S_j| \geq L\cdot d$. However, by the definition of $\de(\cF, \beta)$, since each element of a subsequence $S_j$ is $\beta$-independent of its predecessors, we must have $|S_j|\le d, \forall j\in[L]$ and therefore, $\sum|S_j| \leq L\cdot d$.
In this case, $(s_{\tau}, a_{\tau})$ must be $\beta$-dependent on all subsequences.

Now consider the subsequence $S_{\beta}:=\big((s_{i_1}^{n_1}, a_{i_1}^{n_1}),\dots,(s_{i_\tau}^{n_\tau}, a_{i_\tau}^{n_\tau})\big)$ of $S^{N+1}_K$ which consists of all elements such that $w_n\big((s_i^{n+1}, a_i^{n+1})\big)\geq \beta$. With that being said, $S_{\beta}$ consists of all sample points where large width occurs from epoch $1$ to epoch $N$. The indices in $S_{\beta}$ are in lexical order and $(s_{i_j}^{n_j}, a_{i_j}^{n_j})$ denotes the $j_{\text{th}}$ element in $S_{\beta}$. As we have established, each $(s_{i_j}^{n_j}, a_{i_j}^{n_j})$ is $\beta$-dependent on fewer than $2\epsilon^2/\beta^2 + 8W^2K/\beta^2$ disjoint subsequences of $S^{n_j}_{i_j-1}$ (recall the definition in Equation \eqref{eq:seq}). It follows that each $(s_{i_j}^{n_j}, a_{i_j}^{n_j})$ is $\beta$-dependent on fewer than $2\epsilon^2/\beta^2 + 8W^2K/\beta^2$ disjoint subsequences of $((s_{i_1}^{n_1}, a_i^{n_1}),\dots,(s_{i_{j-1}}^{n_{j-1}}, a_{i_{j-1}}^{n_{j-1}}))\subset S_{\beta}$, i.e., the elements in $S_{\beta}$ before $(s_{i_j}^{n_j}, a_{i_j}^{n_j})$. Combining this with the fact we have established that there exists some $(s_{i_j}^{n_j},a_{i_j}^{n_j})$ that is $\beta$-dependent on at least $\tau/d-1$ disjoint subsequences of $((s_{i_1}^{n_1}, a_i^{n_1}),\dots,(s_{i_{j-1}}^{n_{j-1}}, a_{i_{j-1}}^{n_{j-1}}))$, we have $\tau/d-1\leq 2\epsilon^2/\beta^2 + 8W^2K/\beta^2$. It follows that $\tau\leq \big(2\epsilon^2/\beta^2 + 8W^2K/\beta^2 + 1\big)\cdot d$, which is Equation \eqref{eq:wbound}.

Combining all above results, with probability at least $1-N\delta$,
\begin{align}
    \sum_{n=1}^N \Big(V^{{\pi}^{n+1}}_{b^n} - V^{\pi^{n+1}}\Big) \leq  \frac{2\epsilon^2+8KW^2+\beta^2}{(1-\gamma)\beta^2K}\cdot\de(\cF, \beta) + \frac{N}{1-\gamma}\sqrt{\frac{\log(2/\delta)}{2K}}.
\end{align}
\end{proof}
Next we prove the last step in Section \ref{sec:proofidea}. For notation brevity, we focus on a specific epoch $n$ and drop the dependence on $n$ in the policy and critic functions. We define
\begin{align}\label{eq:spi-Ahat}
    \widehat{A}^t_{b^n}(s,a):= f_t(s,a)+b^n(s,a) - \mathbb{E}_{a'\sim \pi_t(\cdot|s)} [f_t(s,a') + b^n(s,a')],
\end{align}
where $f_t$ is the output of the critic fit step at iteration $t$ in epoch $n$. It can be easily verified that $\mathbb{E}_{a\sim \pi_t(\cdot|s)} \widehat{A}^t_{b^n}(s,a)=0$ and the SPI-SAMPLE update in Equation \eqref{eqn:SPI-actor-sample} is equivalent to
\begin{align}\label{eq:spi-hatA-update}
    \pi_{t+1}(\cdot|s)\propto \pi_t(\cdot|s)\exp\big(\eta \widehat{A}^t_{b^n}(s,\cdot)\indict\{s\in\cK^n\}\big), ~~~\forall s\in\cS.
\end{align}
$\widehat{A}^t_{b^n}$ is indeed our approximation to the true advantage function $A^t_{b^n}$. In the sequel, we show that the actor-critic convergence is upper bounded by the approximation error which can further be controlled with sufficient samples under our assumptions.

\begin{lemma}[SPI-SAMPLE: Actor-Critic Convergence]\label{lemma:SPI-I-Conv}
In ENIAC-SPI-SAMPLE, let $\widehat{A}^t_{b^n}$ be as defined in Equation \eqref{eq:spi-Ahat} and the stepsize $\eta = \sqrt{\frac{\log(|\cA|)}{16W^2T}}$. For any epoch $n\in[N]$, SPI-SAMPLE obtains a sequence of policies $\{\pi_t\}_{t=0}^{T-1}$ such that when comparing to $\tilde{\pi}^n$:
\begin{align}
    \frac{1}{T}\sum_{t=0}^{T-1}(V_{\cM^n}^{\tilde{\pi}^n}-V^{t}_{b^n})&= \frac{1}{T}\sum_{t=0}^{T-1}(V_{\cM^n}^{\tilde{\pi}^n}-V^{t}_{\cM^n})\\
    &\leq\frac{1}{1-\gamma}\Big(8W\sqrt{\frac{\log(|\cA|)}{T}} + \frac{1}{T}\sum_{t=0}^{T-1}\mathbb{E}_{(s,a)\sim \tilde{d}_{\cM^n}}\left[\big(A^t_{b^n}(s,a)-\widehat{A}^t_{b^n}(s,a)\big)\indict\{s\in\cK^n\}\right]\Big).
\end{align}
\end{lemma}
\begin{proof}
The equality is mentioned in Remark \ref{rmk:equiv}. We first show that $A^{t}_{\cM^n}(s,a^\dagger)\leq 0$ for any $s\notin \cK^n$. Since $\pi_t$ uniformly randomly selects an unfamiliar action with bonus $1/(1-\gamma)$ for $s\notin \cK^n$, we have $V^{t}_{\cM^n}(s)\geq 1/(1-\gamma)$. Thus,
\begin{align}
    A^{t}_{\cM^n}(s,a^\dagger) = Q^{t}_{\cM^n}(s,a^\dagger) - V^{t}_{\cM^n}(s) = 1 - (1-\gamma)\cdot V^{t}_{\cM^n}(s) \leq 0, ~~\forall s\notin \cK^n,
\end{align}
where $Q^{t}_{\cM^n}(s,a^\dagger)=1+\gamma V^{t}_{\cM^n}(s)$ ($a^\dagger$ leads $s$ to $s$).
Based on the above result, we have
\begin{align}
    &V^{\tilde{\pi}^n}_{\cM^n} - V^{t}_{\cM^n} = \frac{1}{1-\gamma} \sum_{(s,a)} \tilde{d}_{\cM^n}(s,a)A^{t}_{\cM^n}(s,a)\\
    &=\frac{1}{1-\gamma} \sum_{(s,a)} \tilde{d}_{\cM^n}(s,a)A^{t}_{\cM^n}(s,a)\indict\{s\in\cK^n\} + \frac{1}{1-\gamma} \sum_{(s,a)} \tilde{d}_{\cM^n}(s,a)A^{t}_{\cM^n}(s,a)\indict\{s\notin\cK^n\}\\
    &=\frac{1}{1-\gamma} \sum_{(s,a)} \tilde{d}_{\cM^n}(s,a)A^{t}_{\cM^n}(s,a)\indict\{s\in\cK^n\} + \frac{1}{1-\gamma}\sum_{s} \tilde{d}_{\cM^n}(s)A^{t}_{\cM^n}(s,a^\dagger)\indict\{s\notin\cK^n\}\\
    &\leq \frac{1}{1-\gamma}\sum_{(s,a)}\tilde{d}_{\cM^n}(s,a)A^{t}_{\cM^n}(s,a)\indict\{s\in\cK^n\}\\
    &=\frac{1}{1-\gamma}\sum_{(s,a)}\tilde{d}_{\cM^n}(s,a)A^{t}_{b^n}(s,a)\indict\{s\in\cK^n\}\\
    &= \frac{1}{1-\gamma}\bigg(\mathbb{E}_{(s,a)\sim\tilde{d}_{\cM^n}}\left[\widehat{A}^t_{b^n}(s,a)\indict\{s\in\cK^n\}\right] + \mathbb{E}_{(s,a)\sim\tilde{d}_{\cM^n}}\left[(A^{t}_{b^n}(s,a) - \widehat{A}^t_{b^n}(s,a))\indict\{s\in\cK^n\}\right]\bigg)\nonumber\\
    \label{eq:SPI-I_v_1}
\end{align}
where the first line is by the performance difference lemma in \citet{kakade2003sample}, the third line is due to that $\tilde{\pi}^n$ deterministically picks $a^\dagger$ for $s\notin \cK^n$, and the fifth line follows that $\pi_t$ never picks $a^\dagger$ so for any action $a \in \cA$ we have $A^t_{\cM^n} = A^t_{b^n}$.



Next we establish an upper bound of the first term in Equation \eqref{eq:SPI-I_v_1}. Recall that in SPI-SAMPLE the policy update is equivalent to \eqref{eq:spi-hatA-update}. Thus, for $s\in\cK^n$, we have
\begin{align}
    \KL\big(\tilde{\pi}^n(\cdot|s), \pi_{t+1}(\cdot|s)\big) - \KL\big(\tilde{\pi}^n(\cdot|s), \pi_{t}(\cdot|s)\big) = \mathbb{E}_{a\sim\tilde{\pi}^n(\cdot|s)}[-\eta\widehat{A}^t_{b^n}(s,a)+\log(z^t(s))],
\end{align}
where $z^t(s) := \sum_{a}\pi_t(a|s)\exp(\eta\widehat{A}^t_{b^n}(s,a))$. Since $|\widehat{A}^t_{b^n}(s,a)|\leq 4W$ and when $T>\log(|\cA|)$, $\eta< 1/(4W)$, we have $\eta\widehat{A}^t_{b^n}(s,a)\leq 1$. By the inequality that $\exp(x)\leq 1+x+x^2$ for $x\leq 1$ and $\log(1+x)\leq x$ for $x>-1$,
\begin{align*}
    \log(z^t(s))\leq \eta\sum_{a}\pi_t(a|s)\widehat{A}^t_{b^n}(s,a) + 16\eta^2W^2 = 16\eta^2W^2.
\end{align*}
Hence, for $s\in\cK^n$,
\begin{align}
\KL(\tilde{\pi}^n(\cdot|s), \pi_{t+1}(\cdot|s)) - \KL(\tilde{\pi}^n(\cdot|s), \pi_{t}(\cdot|s)) \leq -\eta\mathbb{E}_{a\sim\tilde{\pi}^n(\cdot|s)}[\widehat{A}^t_{b^n}(s,a)] + 16\eta^2W^2.
\end{align}
Adding both sides from $t=0$ to $T-1$ and taking $\eta=\sqrt{\frac{\log(|\cA|)}{16W^2T}}$, we get
\begin{align}
    &\sum_{t=0}^{T-1} \mathbb{E}_{(s,a)\sim\tilde{d}_{\cM^n}}[\widehat{A}^t_{b^n}(s,a)\indict\{s\in\cK^n\}]\\
    =& \sum_{t=0}^{T-1} \frac{1}{\eta}\mathbb{E}_{s\sim\tilde{d}_{\cM^n}} \left[\Big(\KL(\tilde{\pi}^n(\cdot|s), \pi_{0}(\cdot|s)) - \KL(\tilde{\pi}^n(\cdot|s), \pi_{T}(\cdot|s))\Big)\indict\{s\in\cK^n\}\right] + 16\eta TW^2\\
    \leq &\log(|\cA|)/\eta + 16\eta TW^2 \leq 8W\sqrt{\log(|\cA|)T},
\end{align}
where the inequality follows that $\pi_0(\cdot|s)=\unif(\cA)$. Lastly, combining with Equation \eqref{eq:SPI-I_v_1}, the regret on $\cM^n$ satisfies
\begin{align}
    &\sum_{t=0}^{T-1} (V^{\tilde{\pi}^n}_{\cM^n}-V^{t}_{\cM^n})
    \leq \frac{1}{1-\gamma}\bigg(8W\sqrt{\log(|\cA|)T} + \sum_{t=1}^T \mathbb{E}_{(s,a)\sim\tilde{d}_{\cM^n}}\left[\big(A^{t}_{b^n}(s,a)-\widehat{A}^t_{b^n}(s,a)\big) \indict\{s\in\cK^n\}\right]\bigg).
\end{align}
\end{proof}
Next, we analyze the approximation error and build an upper bound on $A^t_{b^n}-\widehat{A}^t_{b^n}$. Recall that $A^t_{b^n}$ is the true advantage of policy $\pi^n_t$ in the bonus-added MDP and $\widehat{A}^t_{b^n}$ is an approximation to $A^t_{b^n}$ with the empirical minimizer $f_t$ as defined in \eqref{eq:spi-Ahat}. We still focus on a specific epoch $n$ and simplify the notation $\tilde{f}^n_t$ as defined in \eqref{eq:tildef} to $f^*_t$.

\begin{lemma}[SPI-SAMPLE: Approximation Bound]\label{lemma:SPI-I-Approx}
At epoch $n$, assume for all $0\leq t\leq T-1$:
\cut{suppose the best on-policy fit is
\begin{align}
    f^*_t\in\argmin_{f\in\cF}\left[L(f; \rho^n_{\text{cov}}, Q^t_{b^n}-b^n) := \mathbb{E}_{(s,a)\sim \rho^n_{\text{cov}}}(Q^t_{b^n}(s,a) - b^n(s,a) - f(s,a))^2\right].
\end{align}}
\begin{align}\label{eq:spi-stat}
    L(f_t;\rho^n_{\text{cov}},Q^t_{b^n}-b^n)\leq L(f_t^*;\rho^n_{\text{cov}}, Q^t_{b^n}-b^n) + \epsilon_{\text{stat}},
\end{align}
where $\epsilon_{\text{stat}}>0$ is to be determined in the next lemma, and let
\begin{align}\label{eq:epsilon}
\epsilon^2 = NK\big(C\cdot\epsilon_{\text{stat}} + \epsilon_0 + 16W\epsilon_1\big) + 8W^2\log(\cN(\cF, \epsilon_1)/\delta)\cdot\sqrt{NK},
\end{align}
where $\epsilon$ is used in bonus function (see Section \ref{sec:bonus}) and $C$, $\epsilon_0$ are defined in Assumption \ref{ass:SPI-quadra}, and $\epsilon_1>0$ denotes the function cover radius which will be determined later. Under Assumption \ref{ass:SPI-boundedtransfer} and \ref{ass:SPI-quadra}, we have that for every $0\leq t\leq T-1,$ with probability at least $1-\delta$,
\begin{align}
    \mathbb{E}_{(s,a)\sim \tilde{d}_{\cM^n}}\Big(A^t_{b^n}(s,a) - \widehat{A}^t_{b^n}(s,a)\Big)\indict\{s\in\cK^n\} \leq 4\sqrt{|\cA|\epsilon_{\text{bias}}} + 2\beta.
\end{align}
\end{lemma}
\begin{proof}
To analyze the difference between $A^t_{b^n}$ and $\widehat{A}^t_{b^n}$, we introduce an intermediate variable $A^*_t(s,a):=f^*_t +b^n- \mathbb{E}_{a'\sim \pi_t(\cdot|s)}[f^*_t+b^n]$, i.e., the approximated advantage generated by the selected best on-policy fit. Then
\begin{align}
    &\mathbb{E}_{(s,a)\sim \tilde{d}_{\cM^{n}}}(A^t_{b^n} - \widehat{A}^t_{b^n})\indict\{s\in\cK^n\}
    = \mathbb{E}_{(s,a)\sim \tilde{d}_{\cM^n}} \left[(A^t_{b^n} - A_t^* )\indict\{s\in\cK^n\} +( A_t^* - \widehat{A}^t_{b^n} )\indict\{s\in\cK^n\}\right].
\end{align}

For the first difference, we have
\begin{align}\label{eq:VBT_1_term}
    &\mathbb{E}_{(s,a)\sim \tilde{d}_{\cM^n}} \Big(A^t_{b^n} - A_t^* \Big)\indict\{s\in\cK^n\}\\
    &= \mathbb{E}_{(s,a)\sim \tilde{d}_{\cM^n}} \Big(Q^t_{b^n} - f_t^* -b^n\Big)\indict\{s\in\cK^n\} - \mathbb{E}_{s\sim \tilde{d}_{\cM^n}, a\sim \pi_t(\cdot|s)} ( Q^t_{b^n}-f_t^*-b^n)\indict\{s\in \cK^n\}\\
    &\leq \sqrt{\mathbb{E}_{(s,a)\sim \tilde{d}_{\cM^n}}(Q^t_{b^n} - f_t^*-b^n)^2 \indict\{s\in\cK^n\} } + \sqrt{\mathbb{E}_{s\sim \tilde{d}_{\cM^n}, a\sim \pi_t(\cdot|s)}(Q^t_{b^n} - f_t^*-b^n)^2 \indict\{s\in\cK^n\}}\\
    &\leq \sqrt{\mathbb{E}_{(s,a)\sim d^{\tilde{\pi}}}(Q^t_{b^n} - f_t^*-b^n)^2 \indict\{s\in\cK^n\} } + \sqrt{\mathbb{E}_{s\sim d^{\tilde{\pi}}, a\sim \pi_t(\cdot|s)}(Q^t_{b^n} - f_t^*-b^n)^2 \indict\{s\in\cK^n\}}\\
   &= \sqrt{\mathbb{E}_{(s,a)\sim \tilde{d}} ~|\cA|\tilde{\pi}(a|s)\cdot(Q^t_{b^n} - f_t^*-b^n)^2 \indict\{s\in\cK^n\} } + \sqrt{\mathbb{E}_{(s,a)\sim \tilde{d}}~|\cA|\pi_t(a|s)\cdot(Q^t_{b^n} - f_t^*-b^n)^2 \indict\{s\in\cK^n\}}\\
    &< 4\sqrt{|\cA|\epsilon_{\text{bias}}},
\end{align}
where the first inequality is by Cauchy-Schwarz, the second inequality is by Lemma \ref{lemma:dbound}, and the last two lines follow Assumption \ref{ass:SPI-boundedtransfer} and the definition of $f^*_t$.

For the second difference,
\begin{align}
    &\mathbb{E}_{(s,a)\sim \tilde{d}_{\cM^n}} ( A_t^* - \widehat{A}^t_{b^n})\indict\{s\in\cK^n\}\\
    = &\mathbb{E}_{(s,a)\sim \tilde{d}_{\cM^n}} ( f_t^* - f_t)\indict\{s\in\cK^n\} - \mathbb{E}_{s\sim \tilde{d}_{\cM^n}, a\sim \pi_t(\cdot|s)} ( f_t^* - f_t)\indict\{s\in\cK^n\}\label{eq:VBT_2_term}
\end{align}
Next we show that $\Delta f_t:= (f_t^*-f_t)\in\tilde{\cF}^n$. Recall that $\tilde{\cF}^n:=\{\Delta f\in\Delta \cF~|~\|
\Delta f\|_{\cZ^n}\leq \epsilon\}$. We only need to show that $\|\Delta f_t\|_{\cZ^n}\leq\epsilon$. To achieve this, we plan to utilize the fact that $f_t$ is trained with samples generated from $\rho^n_{\text{cov}}:=\unif(d^{\pi^1}_{s_0}, d^{\pi^2}_{s_0}, \dots, d^{\pi^n}_{s_0})$ while $\cZ^n$ is sequentially constructed with samples from $d^{\pi^i}_{s_0}, i\in[n]$. However, such a correlation does not guarantee a trivial concentration bound. We need to deal with the subtle randomness dependency therein: 1. $\pi^i$ depends on $\pi^{[i-1]}$ thus the samples in $\cZ^n$ are not independent; 2. $\cZ^n$ determines $\tilde{\cF}^n$, $\tilde{\cF}^n$ defines the bonus $b^n$, and $\Delta f_t$ is obtained based on $b^n$. So $\Delta f_t$ and $\cZ^n$ are not independent. Nevertheless, we carefully leverage function cover on $\Delta \cF$ to establish a martingale convergence on every anchor function in the cover set, then transform to a bound on the realization $\Delta f_t$.

Let $\cC(\Delta \cF, 2\epsilon_1)$ be a cover set of $\Delta \cF$. Then for every $\Delta f\in\Delta \cF$, there exists a $\Delta g \in \cC(\Delta \cF, 2\epsilon_1)$ such that $\|\Delta f-\Delta g\|_{\infty}\leq 2\epsilon_1$. We rank the samples in $\cZ^n$ in lexical order, i.e., $(s^i_k, a^i_k)$ is the $k_{\text{th}}$ sample generated following $d^{\pi^i}_{s_0}$ at the beginning of the $i_{\text{th}}$ epoch. There are in total $nK$ samples in $\cZ^n$. For every $\Delta g\in \cC(\Delta \cF, 2\epsilon_1)$, we define $nK$ corresponding random variables:
\begin{align}
    X^{\Delta g}_{(i,k)} := (\Delta g(s^i_k, a^i_k))^2 - \mathbb{E}_{(s,a)\sim d^{\pi^i}_{s_0}}[(\Delta g(s, a))^2], ~~i\in[n], k\in[K]
\end{align}
We rank $\{X^{\Delta g}_{(i,k)}\}$ in lexical order and upon which, we define a martingale:
\begin{align}
    Y^{\Delta g}_{0,0} = 0, ~~~Y^{\Delta g}_{(i,k)} = \sum_{(i',k')=(1,1)}^{(i,k)} X^{\Delta g}_{(i',k')}, \quad i\in[n], k\in[K].
\end{align}
Then by single-sided Azuma-Hoeffding's inequality, with probability at least $1-\delta$, for all $\Delta g\in \cC(\Delta \cF, 2\epsilon_1)$, it holds that
\begin{align}\label{eq:deltag}
    Y^{\Delta g}_{(n,K)} \leq \sqrt{32W^4\cdot nK\cdot \log\Big(\frac{\cN(\Delta \cF,2\epsilon_1)}{\delta}\Big)}
    \leq \sqrt{64W^4\cdot nK\cdot\log\Big(\frac{\cN( \cF,\epsilon_1)}{\delta}\Big)},
\end{align}
where the right inequality is by Lemma \ref{lemma:cover}. Next, we transform to $\Delta f_t$. Since there exists a $\Delta g\in\cC(\Delta\cF,2\epsilon_1)$ such that $\|\Delta f_t - \Delta g\|_{\infty}\leq 2\epsilon_1$, we have that for all $i\in[n]$ and $k\in[K]$,
\begin{align}
   &\left| (\Delta f_t(s^{i}_{k}, a^{i}_{k}))^2 - (\Delta g(s^{i}_{k}, a^{i}_{k}))^2\right| \\
   &= |\Delta f_t(s^{i}_{k}, a^{i}_{k}) - \Delta g(s^{i}_{k}, a^{i}_{k})|\cdot|\Delta f_t(s^{i}_{k}, a^{i}_{k}) + \Delta g(s^{i}_{k}, a^{i}_{k}))|\leq 8W\epsilon_1
\end{align}
and
\begin{align}
&\left| \mathbb{E}_{(s,a)\sim d^{\pi^{i}}_{s_0}}[(\Delta f_t(s, a))^2]-\mathbb{E}_{(s,a)\sim d^{\pi^{i}}_{s_0}}[(\Delta g(s, a))^2]\right|\\
&\leq \mathbb{E}_{(s,a)\sim d^{\pi^{i}}_{s_0}} |\Delta f_t(s, a)-\Delta g(s, a)|\cdot |\Delta f_t(s, a)+\Delta g(s, a)|\leq 8W\epsilon_1
\end{align}
Therefore,
\begin{align}\label{eq:fgcompare}
  Y^{\Delta f_t}_{(n,K)} &=  \sum_{(i,k)=(1,1)}^{(n,K)} (\Delta f_t(s^{i}_{k}, a^{i}_{k}))^2 - \mathbb{E}_{(s,a)\sim d^{\pi^{i}}_{s_0}}[(\Delta f_t(s, a))^2]\\
  &\leq \sum_{(i,k)=(1,1)}^{(n,K)} (\Delta g(s^{i}_{k}, a^{i}_{k}))^2 - \mathbb{E}_{(s,a)\sim d^{\pi^{i}}_{s_0}}[(\Delta g(s, a))^2] + nK\cdot 16W\epsilon_1\\
  &=Y^{\Delta g}_{(n,K)} +nK\cdot 16W\epsilon_1.
\end{align}
Note that
\begin{align}\label{eq:Ytof}
Y^{\Delta f_t}_{(n,K)}
&=\|\Delta f_t\|^2_{\cZ^n}- \sum_{i=1}^n K\cdot \mathbb{E}_{d^{\pi^i}_{s_0}}[(\Delta f_t)^2]  =  \|\Delta f_t\|^2_{\cZ^n}-nK\cdot \mathbb{E}_{\rho^n_{\text{cov}}}[(\Delta f_t)^2].
\end{align}
Combining \eqref{eq:deltag}, \eqref{eq:fgcompare}, and \eqref{eq:Ytof}, we have that
\begin{align}
    \|\Delta f_t\|^2_{\cZ^n}
    &\leq nK\cdot \mathbb{E}_{\rho^n_{\text{cov}}}[(\Delta f_t)^2] + nK\cdot 16W\epsilon_1 + \sqrt{64W^4\cdot nK\cdot\log\Big(\frac{\cN( \cF,\epsilon_1)}{\delta}\Big)}.
\end{align}
\cut{
For every $\Delta g\in \cC(\Delta \cF, 2\epsilon_1)$, we construct a martingale:
\begin{align}
    X^1_1, X^1_2,\dots,X^1_K, X^2_1, \dots X^2_K,\dots, X^n_1,\dots X^n_K,
\end{align}
where $X^i_j = \sum_{(s,a)\in\cZ^{i-1}}(\Delta g(s,a))^2$

\cut{we construct $n$ events $\{\cE^m_g\}_{m=1}^n$ where $ \cE^m_g$ denotes the event:
\begin{align}
   \left|\frac{\sum_{(s,a)\in\cD^m}\big(\Delta g(s,a)\big)^2}{|\cD^m|}-\mathbb{E}_{(s,a)\sim d^m} (\Delta g)^2 \right|
    \leq \sqrt{\frac{16W^4\log(\frac{2\cN( \cF,\epsilon_1)}{\delta})}{K}}.
\end{align}}

Define $h^m(\Delta f_t):=\frac{1}{K}\sum_{(s,a)\in\cD^m} \big(\Delta f_t(s,a)\big)^2$ for $m\in[n]$. For each $m\in[n]$, with probability at least $1-\delta$, for all $\Delta g \in\cC(\Delta\cF, 2\epsilon_1)$,
\begin{align}
  | h^m(\Delta g)-\mathbb{E}_{(s,a)\sim d^m} (\Delta g)^2 |\leq \sqrt{\frac{8W^4\log(\frac{2\cN(\Delta \cF,2\epsilon_1)}{\delta})}{K}}\leq \sqrt{\frac{16W^4\log(\frac{2\cN( \cF,\epsilon_1)}{\delta})}{K}},
\end{align}
where the rightmost inequality is by Lemma \ref{lemma:cover}. We define the above event to be $\cE^{m}$. Conditioning on $\cE^m$, we have for all $\Delta f \in\Delta\cF$, there exists $\Delta g\in\cC(\Delta\cF,\epsilon_1)$ such that $\|\Delta f-\Delta g\|_{\infty}\leq 2\epsilon_1$ and
\begin{align}
    | h^m(\Delta f)-\mathbb{E}_{(s,a)\sim d^m} (\Delta f)^2 | &\leq | h^m(\Delta g)-\mathbb{E}_{(s,a)\sim d^m} (\Delta g)^2 | +16W\epsilon_1 \\
    &\leq \sqrt{\frac{16W^4\log(\frac{2\cN(\cF, \epsilon_1)}{\delta})}{K}} + 16W\epsilon_1.
\end{align}
After controlling the difference between the empirical mean and true expectation for every epoch $m\in[n]$, we control the sum over all epochs till $n$. Specifically, we define a martingale $X_0, X_1,\dots, X_n$ as:
\begin{align}
    X_0 = 0, \quad X_m = X_{m-1} + \frac{1}{n}\Big(h^m(\Delta f_t) - \mathbb{E}_{(s,a)\sim d^m} (\Delta f_t)^2 \Big) \quad m=1,\cdots,n.
\end{align}
Conditioning on event $\cE^1\cap\dots\cap\cE^n$, by Azuma-Hoeffding inequality, we have that with probability at least $1-\delta$,
\begin{align}
    |X_n|\leq
    \frac{8W^2\log(\frac{2\cN(\cF, \epsilon_1)}{\delta})}{\sqrt{nK}}+\frac{32W\epsilon_1}{\sqrt{n}}.
\end{align}
Recall that $\rho^n_{\text{cov}} = \unif(d^1, \dots, d^n)$. Hence, $X_n$ is essentially $\|\Delta f_t\|^2_{\cZ^n}/(nK)-\mathbb{E}_{(s,a)\sim \rho^n_{\text{cov}}}[(\Delta f_t)^2]$. Thus,
\begin{align}
    \|\Delta f_t\|^2_{\cZ^n}\leq nK\cdot(\mathbb{E}_{\rho^n_{\text{cov}}}(\Delta f_t)^2) + 8W^2\log\Big(\frac{2\cN(\cF, \epsilon_1)}{\delta}\Big)\cdot\sqrt{nK}+32W\epsilon_1\cdot K\sqrt{n}.
\end{align}}
By Assumption \ref{ass:SPI-quadra},
\begin{align}
    \mathbb{E}_{\rho^n_{\text{cov}}}[(\Delta f_t)^2]&=\mathbb{E}_{(s,a)\sim \rho^n_{\text{cov}}} [(f_t^* - f_t)^2]\leq C \cdot(L(f_t; \rho^n_{\text{cov}}, Q^t_{b^n}-b^n) - L(f_t^*; \rho^n_{\text{cov}}, Q^t_{b^n}-b^n)) + \epsilon_0\\
    &\leq C\cdot \epsilon_{\text{stat}} + \epsilon_0.
\end{align}

By the choice of $\epsilon$, $\|\Delta f_t\|^2_{\cZ^n}\leq \epsilon^2$ with probability at least $1-\delta$. Thus, $\Delta f_t\in \tilde{\cF}^n$ and for all $(s,a)\in\cK^n$, $|f_t^*(s,a)-f_t(s,a)|\leq \beta$. Plugging into \eqref{eq:VBT_2_term}, we have $\eqref{eq:VBT_2_term} \leq 2\beta$. The desired result is obtained.
\end{proof}
Next, we give an explicit form of $\epsilon_{\text{stat}}$ as defined in Equation \eqref{eq:spi-stat}.

\begin{lemma}\label{lemma:stat}
Following the same notation as in Lemma \ref{lemma:SPI-I-Approx}, it holds with probability at least $1-\delta$ that
\begin{align}
    &L(f_t;\rho^n_{\text{cov}},Q^t_{b^n}-b^n)- L(f^*_t;\rho^n_{\text{cov}}, Q^t_{b^n}-b^n)
    \leq \frac{500C\cdot W^4\cdot \log \Big(\frac{\cN(\cF, \epsilon_2)}{\delta}\Big)}{M} + 13W^2\cdot \epsilon_2 + \epsilon_0,
\end{align}
where $C$, $\epsilon_0$ are defined in Assumption \ref{ass:SPI-quadra}, and $\epsilon_2>0$ denotes the function cover radius which will be determined later.
\end{lemma}
\begin{proof}
First note that in the loss function, the expectation has a nested structure: the outer expectation is taken over $(s,a)\sim \rho^n_{\text{cov}}$ and the inner conditional expectation is $Q^t_{b^n}(s,a)=\mathbb{E}^{\pi_t}[\sum_{h=0}^{\infty} \gamma^h \big(r(s_h,a_h)+b^n(s_h,a_h)\big)|(s_0,a_0)=(s,a)]$ given a sample of $(s,a)\sim \rho^n_{\text{cov}}$. To simplify the notation, we use $x$ to denote $(s,a)$, $y|x$ for an unbiased sample of $Q^t_{b^n}(s,a)-b^n(s,a)$, and $\nu$ for $\rho^n_{\text{cov}}$, the marginal distribution over $x$, then the loss function can be recast as
\begin{align}
    &\mathbb{E}_{x\sim\nu}[(f_t(x)-\mathbb{E}[y|x])^2]:= L(f_t;\rho^n_{\text{cov}},Q^t_{b^n}-b^n)\\
    &\mathbb{E}_{x\sim\nu}[(f^*_t(x)-\mathbb{E}[y|x])^2]:= L(f^*_t;\rho^n_{\text{cov}},Q^t_{b^n}-b^n).
\end{align}
In particular, $f_t$ can be rewritten as
\begin{align}
    f_t\in\argmin_{f\in\cF} \sum_{i=1}^M (f(x_i)-y_i)^2,
\end{align}
where $(x_i,y_i)$ are drawn i.i.d.: $x_i$ is generated following the marginal distribution $\nu$ and $y_i$ is generated conditioned on $x_i$.
For any function $f$, we have:
\begin{align*}
  &\mathbb{E}_{x,y}[(f_t(x)-y)^2]\\
  =~&\mathbb{E}_{x,y}[(f_t(x)-\mathbb{E}[y|x])^2] + \mathbb{E}_{x,y}[(\mathbb{E}[y|x] - y)^2] + 2\mathbb{E}_{x,y}[(f_t(x) - \mathbb{E}[y|x])(\mathbb{E}[y|x]-y)]\\
  =~&\mathbb{E}_{x,y}[(f_t(x)-\mathbb{E}[y|x])^2] + \mathbb{E}_{x,y}[(\mathbb{E}[y|x] - y)^2],
\end{align*}
where the last step follows from the cross term being zero. Thus we can rewrite the generalization error as
\begin{align}\label{eq:lossdiff}
  &\mathbb{E}_{x}[(f_t(x)-\mathbb{E}[y|x])^2]-  \mathbb{E}_{x}[(f^*_t(x)-\mathbb{E}[y|x])^2]\\
  =~&\mathbb{E}_{x,y}(f_t(x)-y)^2 - \mathbb{E}_{x,y}(f^*_t(x)-y)^2.
\end{align}

Next, we establish a concentration bound on $f_t$. Since $f_t$ depends on the training set $\{(x_i, y_i)\}_{i=1}^M$, as in Assumption \ref{lemma:SPI-I-Approx}, we use a function cover on $\cF$ for a uniform convergence argument. We denote by $\mathscr{F}^n_t$ the $\sigma$-algebra generated by randomness before epoch $n$ iteration $t$. Recall that $f_t^*\in\argmin_{f\in\cF}L(f;\rho^n_{\text{cov}}, Q^t_{b^n}-b^n)$. Conditioning on $\sF^n_t$, $\rho^n_{\text{cov}}$, $Q^t_{b^n}-b^n$, and $f^*_t$ are all deterministic. For any $f\in\cF$, we define
\begin{align}
Z_i(f):=(f(x_i)-y_i)^2 - (f^*_t(x_i)-y_i)^2, ~~~i\in[M]
\end{align}
Then $Z_1(f), \dots, Z_M(f)$ are i.i.d. random variables and
\begin{align}
    \mathbb{V}[Z_i(f) ~|~ \sF^n_t]&\leq \mathbb{E}[Z_i(f)^2~|~\sF^n_t]\\
    &= \mathbb{E}\left[\Big((f(x_i)-y_i)^2 - (f^*_t(x_i)-y_i)^2\Big)^2~|~\sF^n_t\right]\\
    &=\mathbb{E}\left[\big(f(x_i)-f_t^*(x_i)\big)^2\cdot \big(f(x_i)+f^*_t(x_i)-2y_i\big)^2~|~\sF^n_t\right]\\
    &\leq 36W^4\cdot \mathbb{E}[\big(f(x_i)-f_t^*(x_i)\big)^2~|~\sF^n_t]\\
    &\leq 36W^4 \cdot (C\cdot\mathbb{E}[Z_i(f)~|~\sF^n_t]+\epsilon_0),
\end{align}
where the last inequality is by Assumption \ref{ass:SPI-quadra} and Equation \eqref{eq:lossdiff}. Next, we apply Bernstein's inequality on the function cover $\cC(\cF, \epsilon_2)$ and take the union bound. Specifically, with probability at least $1-\delta$, for all $g\in\cC(\cF, \epsilon_2)$,
\begin{align}\label{eq:bernstein}
    &\mathbb{E}[Z_i(g)~|~\sF^n_t] - \frac{1}{M}\sum_{i=1}^M Z_i(g)\\
    \leq &\sqrt{\frac{2\mathbb{V}[Z_i(g) ~|~\sF^n_t]\cdot \log \frac{\cN(\cF, \epsilon_2)}{\delta}}{M}} + \frac{12W^4\cdot\log \frac{\cN(\cF, \epsilon_2)}{\delta}}{M}\\
    \leq &\sqrt{\frac{72W^4 (C\cdot\mathbb{E}[Z_i(g)~|~\sF^n_t]+\epsilon_0)\cdot \log \frac{\cN(\cF, \epsilon_2)}{\delta}}{M}} + \frac{12W^4\cdot\log \frac{\cN(\cF, \epsilon_2)}{\delta}}{M}.
\end{align}
For $f_t$, there exists $g\in\cC(\cF, \epsilon_2)$ such that $\|f_t-g\|_{\infty}\leq \epsilon_2$ and
\begin{align}
    \left|Z_i(f_t) - Z_i(g)\right|& = \left| (f_t(x_i)-y_i)^2-(g(x_i)-y_i)^2\right| \\
    & = \left| f_t(x_i)-g(x_i)\right|\cdot \left| f_t(x_i)+g(x_i)-2y_i\right|
    \leq 6W^2\epsilon_2.
\end{align}
Therefore, with probability at least $1-\delta$,
\begin{align}
    &\mathbb{E}[Z_i(f_t)~|~\sF^n_t] - \frac{1}{M}\sum_{i=1}^M Z_i(f_t)\\
    \leq & \mathbb{E}[Z_i(g)~|~\sF^n_t] - \frac{1}{M}\sum_{i=1}^M Z_i(g) + 12W^2\epsilon_2\\
    \leq &\sqrt{\frac{72W^4 (C\cdot\mathbb{E}[Z_i(g)~|~\sF^n_t]+\epsilon_0)\log \frac{\cN(\cF, \epsilon_2)}{\delta}}{M}} + \frac{12W^4\log \frac{\cN(\cF, \epsilon_2)}{\delta}}{M} + 12W^2\epsilon_2\\
    \leq &\sqrt{\frac{72W^4 (C\cdot\mathbb{E}[Z_i(f_t)~|~\sF^n_t]+6CW^2\epsilon_2+\epsilon_0) \log \frac{\cN(\cF, \epsilon_2)}{\delta}}{M}} + \frac{12W^4\log \frac{\cN(\cF, \epsilon_2)}{\delta}}{M} + 12W^2\epsilon_2.
\end{align}
Since $f_t$ is an empirical minimizer, we have $\frac{1}{M}\sum_{i=1}^M Z_i(f_t)\leq 0$. Thus,
\begin{align}
    \mathbb{E}[Z_i(f_t)~|~\sF^n_t]\leq &\sqrt{\frac{72W^4 (C\cdot\mathbb{E}[Z_i(f_t)~|~\sF^n_t]+6CW^2\epsilon_2+\epsilon_0)\log \frac{\cN(\cF, \epsilon_2)}{\delta}}{M}} + \frac{12W^4\log \frac{\cN(\cF, \epsilon_2)}{\delta}}{M} + 12W^2\epsilon_2.
\end{align}
Solving the above inequality with quadratic formula and using $\sqrt{a+b}\leq\sqrt{a}+\sqrt{b}$, $\sqrt{ab}\leq a/2 + b/2$ for $a>0,b>0$, we obtain
\begin{align}
    \mathbb{E}[Z_i(f_t)~|~\sF^n_t]
    \leq &\frac{500C\cdot W^4\cdot \log \frac{\cN(\cF, \epsilon_2)}{\delta}}{M} + 13W^2\cdot \epsilon_2 + \epsilon_0.
\end{align}
Since the right-hand side is a constant, through taking another expectation, we have
\begin{align}
    \mathbb{E}[Z_i(f_t)]\leq &\frac{500C\cdot W^4\cdot \log \frac{\cN(\cF, \epsilon_2)}{\delta}}{M} + 13W^2\cdot \epsilon_2 + \epsilon_0.
\end{align}
Notice that $ \mathbb{E}[Z_i(f_t)]=L(f_t;\rho^n_{\text{cov}}, Q^t_{b^n}-b^n) - L(f^*_t;\rho^n_{\text{cov}}, Q^t_{b^n}-b^n)$. The desired result is obtained.
\end{proof}

Combining all previous lemmas, we have the following theorem which states the detailed sample complexity of ENIAC-SPI-SAMPLE (a detailed version of Theorem \ref{thm:SPI-Sample})
\begin{theorem}[Main Result: Sample Complexity of ENIAC-SPI-SAMPLE]\label{thm:SPI-SAMPLE-Detail} Let $\delta\in(0, 1)$ and $\varepsilon\in(0, 1/(1-\gamma))$. With Assumptions \ref{ass:SPI-boundedtransfer}, \ref{ass:SPI-quadra}, \ref{ass:W}, and \ref{ass:spi-cover}, we set the hyperparameters as:
\begin{align}
    \beta &=\frac{\varepsilon(1-\gamma)}{2},
    ~~T=\frac{64W^2\cdot\log|\cA|}{\varepsilon^2(1-\gamma)^2},
    ~~N\geq \frac{32W^2\cdot\de(\cF, \beta)}{\varepsilon^3(1-\gamma)^3},
    ~~\eta = \sqrt{\frac{\log(|\cA|)}{16W^2T}}\\
    \epsilon_1 &= \frac{(1-\gamma)^3\varepsilon^3}{128W\cdot\de(\cF, \beta)},
    ~~K = \frac{128W^2\cdot \de(\cF, \beta)\cdot\big(\log(\frac{3NT\cdot\cN(\cF, \epsilon_1)}{\delta})\big)^2\cdot\log(\frac{6NT}{\delta})}{\varepsilon^3(1-\gamma)^3},\\
    ~~\epsilon_2 &= \frac{(1-\gamma)^3\varepsilon^3}{110C\cdot W^2\cdot\de(\cF, \beta)},
    ~~M = \frac{4000C^2W^4\cdot \de(\cF, \beta)\cdot\log(\frac{3NT\cdot\cN(\cF, \epsilon_2)}{\delta})}{\varepsilon^3(1-\gamma)^3},\\
     ~
\end{align}
and $\epsilon$ satisfies Equation \eqref{eq:epsilon} correspondingly. Then with probability at least $1-\delta$, for the average policy $\pi^N_{\text{ave}}:=\pi^N_{\text{ave}}:=\unif(\pi^2, \dots, \pi^{N+1})$, we have
\begin{align}
V^{\pi^N_{\text{ave}}} \geq V^{\tilde{\pi}} - \frac{4\sqrt{|\cA|\epsilon_{\text{bias}}}}{1-\gamma}-\epsilon_0\cdot \frac{16C\de(\cF, \beta)}{\varepsilon^2(1-\gamma)^3} - 9\varepsilon
\end{align}
for any comparator $\tilde{\pi}$ with total number of samples:
\begin{align}
  \widetilde{O}\Big( \frac{C^2W^8\cdot\big(\de(\cF, \beta)\big)^2\cdot \big(\log(\cN(\cF,  \epsilon'))\big)^2}{\varepsilon^{8}(1-\gamma)^{8}}\Big),
\end{align}
where $\epsilon'=\min(\epsilon_1, \epsilon_2)$. 
\end{theorem}
\begin{proof}
By Lemma \ref{lemma:SPI-I-Bonus}, we have that with probability at least $1-N\delta_1$,
\begin{align}\label{eq:thm1-1}
 V^{\pi^N_{\text{ave}}} \geq \frac{1}{N}\sum_{n=1}^N V^{\pi^{n+1}}_{b^n} - 
 & \frac{2\epsilon^2+8KW^2+\beta^2}{(1-\gamma)\beta^2NK}\cdot\de(\cF, \beta) + \frac{1}{1-\gamma}\sqrt{\frac{\log(2/\delta_1)}{2K}}.
\end{align}
By Lemma \ref{lemma:SPI-I-Conv}, \ref{lemma:SPI-I-Approx}, and \ref{lemma:Vpibound}, we have that for every $n\in[N]$, with probability at least $1-2T\delta_1$,
\begin{align}
V^{\pi^{n+1}}_{b^n} \geq V^{\tilde{\pi}} - \frac{1}{1-\gamma}\Big( 8W\sqrt{\frac{\log(|\cA|)}{T}} + 4\sqrt{|\cA|\epsilon_{\text{bias}}} + 2\beta\Big).\label{eq:thm1-2}
\end{align}
Combining inequalities \eqref{eq:thm1-1} and \eqref{eq:thm1-2}, we have with probability at least $1-3NT\delta_1$,
\begin{align}
 V^{\pi^N_{\text{ave}}} \geq V^{\tilde{\pi}} &- \frac{1}{1-\gamma}\bigg(
   \frac{2\epsilon^2+8KW^2+\beta^2}{\beta^2NK}\cdot\de(\cF, \beta)
    + \sqrt{\frac{\log(2/\delta_1)}{2K}}\\
    &+ 8W\sqrt{\frac{\log(|\cA|)}{T}}
    + 4\sqrt{|\cA|\epsilon_{\text{bias}}}
    + 2\beta
    \bigg).\label{eq:valuefinal}
    \end{align}
We plug in the value of $\epsilon^2$ in Equation \eqref{eq:epsilon} with the bound on $\epsilon_{\text{stat}}$ in Lemma \ref{lemma:stat} and choose hyperparameters such that every term in \eqref{eq:valuefinal} (except for the ones with $\epsilon_0$ or $\epsilon_{\text{bias}}$) is bounded by $\varepsilon$. Finally, we set $\delta_1=\delta/(3NT)$ and $\epsilon'=\min(\epsilon_1, \epsilon_2)$.
In total, the sample complexity is
\begin{align}
   N(K+TM)=\widetilde{O}\Big( \frac{C^2W^8\cdot\big(\de(\cF, \beta)\big)^2\cdot \big(\log(\cN(\cF,  \epsilon'))\big)^2}{\varepsilon^{8}(1-\gamma)^{8}}\Big).
\end{align}
\end{proof}

\begin{corollary}
If Assumption \ref{ass:eval_close} holds, with proper hyperparameters, the average policy $\pi^N_{\text{ave}}:=\unif(\pi^2,\dots,\pi^{N+1})$ of ENIAC-SPI-SAMPLE achieves $V^{\pi^N_{\text{ave}}}\geq V^{\tilde{\pi}} - \varepsilon$ with probability at least $1-\delta$ and the sample complexity is
\begin{align}
  \widetilde{O}\Big(\frac{W^8\cdot\big(\de(\cF, \beta)\big)^2\cdot\big(\log(\cN(\cF, \epsilon'))\big)^2}{\varepsilon^{8}(1-\gamma)^{8}}\Big).
\end{align}
\end{corollary}
\begin{proof}
The result is straightforward as mentioned in Remark \ref{rmk:spi} that under Assumption \ref{ass:eval_close}, $\epsilon_{\text{bias}}=0$, $C=1$, and $\epsilon_0=0$.
\end{proof}

\subsection{Sample Complexity of ENIAC-SPI-COMPUTE}\label{sec:SPI-COMPUTE-PROOF}
In this section, we prove the result for ENIAC-SPI-COMPUTE. SPI-COMPUTE only differs from SPI-SAMPLE at two places: the value of the bonus and the actor update rule. These differences cause changes in the bonus bound result and the convergence analysis while Lemma \ref{lemma:SPI-I-Approx} and \ref{lemma:stat} still hold with the same definition of $\widehat{A}^t_{b^n}$ as in \eqref{eq:spi-Ahat}. In the sequel, we present the bonus bound and the convergence result for SPI-COMPUTE.
\begin{lemma}[SPI-COMPUTE: The Bound of Bonus]\label{lemma:SPI-II-Bonus}
With probability at least $1-N\delta$,
\begin{align}
    \sum_{n=1}^N V^{{\pi}^{n+1}}_{b^n} - V^{\pi^{n+1}} &\leq \frac{|\cA|}{(1-\gamma)\alpha} \cdot \frac{2\epsilon^2 + 8W^2K + \beta^2}{\beta^2K} \cdot\de(\cF, \beta)+ \frac{N|\cA|}{(1-\gamma)\alpha}\sqrt{\frac{\log(2/\delta)}{2K}}.
\end{align}
\end{lemma}
The proof is similar to Lemma \ref{lemma:SPI-I-Bonus}. We only need to revise the bonus value from $\frac{1}{1-\gamma}$ to $\frac{|\cA|}{(1-\gamma)\alpha}$.

As for the actor-critic convergence, we focus on a specific epoch $n$ and still define
\begin{align}\label{eq:spi-COMPUTE-hatA}
    \widehat{A}^t_{b^n}(s,a):= f_t(s,a)+b^n(s,a) - \mathbb{E}_{a'\sim \pi_t(\cdot|s)} [f_t(s,a') + b^n(s,a')].
\end{align}
It is easy to verify that $\mathbb{E}_{a\sim\pi_t(\cdot|s)}[\widehat{A}^t_{b^n}]=0$ and for $s\in\cK^n$, the actor update in SPI-COMPUTE is equivalent to
\begin{align}
    \pi'_{t+1}(a|s)\propto \pi'_t(a|s)\exp\big(\eta \widehat{A}^t_{b^n}(s,a)\big), ~\pi_{t+1} = (1-\alpha)\pi'_{t+1} + \alpha\unif(\cA)
\end{align}
since $b^n(s,\cdot)=0$ for $s\in\cK^n$. As before, we use $\widehat{A}_{b^n}(s,a)$ to approximate the true advantage of $\pi^n_t$ on $\cM_{b^n}$. Then we have the following result.
\begin{lemma}[SPI-COMPUTE: Actor-Critic Convergence]\label{lemma:SPI-II-Conv}
In ENIAC-SPI-COMPUTE, let $\widehat{A}^t_{b^n}$ be as defined in Equation \eqref{eq:spi-COMPUTE-hatA}, $\eta = \sqrt{\frac{\log(|\cA|)}{16W^2T}}$, and $\alpha=\frac{1}{1+\sqrt{T}}$. For any epoch $n\in[N]$, SPI-COMPUTE obtains a sequence of policies $\{\pi_t\}_{t=0}^{T-1}$ such that when comparing to $\tilde{\pi}^n$:

\begin{align}
    &\frac{1}{T}\sum_{t=0}^{T-1}(V_{\cM^n}^{\tilde{\pi}^n}-V^{t}_{b^n})= \frac{1}{T}\sum_{t=0}^{T-1}(V_{\cM^n}^{\tilde{\pi}^n}-V^{{t}}_{\cM^n})\\
    &\leq\frac{1}{1-\gamma}\Big(12W\sqrt{\frac{\log(|\cA|)}{T}} + \frac{1}{T}\sum_{t=0}^{T-1}\Big(\mathbb{E}_{(s,a)\sim \tilde{d}_{\cM^n}}\big(A^t_{b^n}(s,a)-\widehat{A}^t_{b^n}(s,a)\big)\indict\{s\in\cK^n\}\Big).
\end{align}
\end{lemma}
\begin{proof}[Proof of Lemma \ref{lemma:SPI-II-Conv}]
Similar to the reasoning in Lemma \ref{lemma:SPI-I-Conv}, we first have that $A^{t}_{\cM^n}(s,a^\dagger)\leq 0$ for any $s\notin \cK^n$. To see this, note that for $s\notin\cK^n$, there exists an action with bonus $b^n = |\cA|/\big((1-\gamma)\alpha\big)$ and $\pi_t$ has probability at least $\alpha/|\cA|$ selects that action. Therefore, $V^{t}_{\cM^n}(s)\geq 1/(1-\gamma)$ and
\begin{align}
    A^{t}_{\cM^n}(s,a^\dagger) = Q^{t}_{\cM^n}(s,a^\dagger) - V^{t}_{\cM^n}(s) = 1 - (1-\gamma)\cdot V^{t}_{\cM^n}(s) \leq 0, ~~\forall s\notin \cK^n.
\end{align}
\cut{\begin{align}
    A^{t}_{\cM^n}(s,a^\dagger) = Q^{t}_{\cM^n}(s,a^\dagger) - V^{t}_{\cM^n}(s) = 1 - (1-\gamma)\cdot V^{t}_{\cM^n}(s) \leq 0, ~~\forall s\notin \cK^n.
\end{align}}
Recall that $\tilde{\pi}^n$ deterministically picks $a^\dagger$ for $s\notin \cK^n$. Based on the above inequality, it holds that
\begin{align}
    V^{\tilde{\pi}^n}_{\cM^n} - V^{t}_{\cM^n} &= \frac{1}{1-\gamma} \sum_{(s,a)} \tilde{d}_{\cM^n}(s,a)A^{t}_{\cM^n}(s,a)\leq \frac{1}{1-\gamma}\sum_{(s,a)}\tilde{d}_{\cM^n}(s,a)A^{t}_{\cM^n}(s,a)\indict\{s\in\cK^n\}\\
    &=\frac{1}{1-\gamma}\sum_{(s,a)}\tilde{d}_{\cM^n}(s,a)A^{t}_{b^n}(s,a)\indict\{s\in\cK^n\}.\label{eq:SPI-II_v_1}
\end{align}
Next we restrict on $s\in\cK^n$ and establish the consecutive KL difference on $\{\pi'_t(\cdot|s)\}$. Specifically, since for $s\in\cK^n$, $\pi'_{t+1}(\cdot|s)\propto \pi'_t(\cdot|s)\exp(\eta \widehat{A}^t_{b^n}(s,a)),$
\begin{align}
    \KL(\tilde{\pi}^n(\cdot|s), \pi'_{t+1}(\cdot|s)) - \KL(\tilde{\pi}^n(\cdot|s), \pi'_{t}(\cdot|s)) = \mathbb{E}_{a\sim\tilde{\pi}^n(\cdot|s)}[-\eta\widehat{A}^t_{b^n}(s,a)+\log(z^t)],
\end{align}
where $z^t := \sum_{a}\pi'_t(a|s)\exp(\eta\widehat{A}^t_{b^n}(s,a))$. With the assumptions that $|\widehat{A}^t_{b^n}(s,a)|\leq 4W$ and $\eta\leq 1/(4W)$ when $T>\log(|\cA|)$, we have that $\eta\widehat{A}^t_{b^n}(s,a)\leq 1$. By the inequality that $\exp(x)\leq 1+x+x^2$ for $x\leq 1$, we have that
\begin{align}
    \log(z^t)&\leq \log(1 + \eta\sum_{a}\pi'_t(a|s)\widehat{A}^t_{b^n}(s,a) + 16\eta^2W^2)\\
    &=\log\bigg(1 + \eta\sum_{a} \Big(\frac{\pi_t(a|s)}{1-\alpha} - \frac{\alpha\cdot \unif(\cA)}{1-\alpha}\Big) \cdot \widehat{A}^t_{b^n}(s,a)+ 16\eta^2 W^2\bigg)\\
    &=\log\bigg(1 - \frac{\eta\alpha}{(1-\alpha)|\cA|}\sum_{a}  \widehat{A}^t_{b^n}(s,a)+ 16\eta^2 W^2\bigg)\\
    & \leq \log(1+\eta\frac{4W\alpha}{1-\alpha} +16\eta^2W^2)\\
    & \leq \frac{4W\eta\alpha}{1-\alpha} + 16\eta^2W^2,
\end{align}
where the second line follows from that $\pi_t'=\frac{\pi_t}{1-\alpha}-\frac{\alpha \unif(\cA)}{1-\alpha}$ and the last line follows that $\log(1+x)\leq x$ for $x>0$.
Hence, for $s\in\cK^n$,
\begin{align}
\KL(\tilde{\pi}^n(\cdot|s), \pi'_{t+1}(\cdot|s)) - \KL(\tilde{\pi}^n(\cdot|s), \pi'_{t}(\cdot|s)) \leq -\eta\mathbb{E}_{a\sim\tilde{\pi}^n(\cdot|s)}[\widehat{A}^t_{b^n}(s,a)] + \frac{4W\eta\alpha}{1-\alpha} + 16\eta^2W^2.
\end{align}
Take $\alpha=\frac{1}{1+\sqrt{T}}$. Adding both sides from $t=0$ to $T-1$, we get
\begin{align}
    &\sum_{t=0}^{T-1} \mathbb{E}_{(s,a)\sim\tilde{d}_{\cM^n}}[\widehat{A}^t_{b^n}(s,a)\indict\{s\in\cK^n\}]\\
    \leq &\frac{1}{\eta}\mathbb{E}_{s\sim \tilde{d}_{\cM^n}}\left[\big(\KL(\tilde{\pi}^n(\cdot|s), \pi'_{0}(\cdot|s)) -\KL(\tilde{\pi}^n(\cdot|s), \pi'_{T}(\cdot|s))\big)\indict\{s\in\cK^n\}\right] +  4W\sqrt{T} + 16\eta TW^2\\
    \leq &\log(|\cA|)/\eta + 4W\sqrt{T} + 16\eta TW^2 \leq 12W\sqrt{\log(|\cA|)T}.
\end{align}
Combining with Equation \eqref{eq:SPI-II_v_1}, the regret on $\cM^n$ satisfies
\begin{align}
    &\sum_{t=0}^{T-1} (V^{\tilde{\pi}^n}_{\cM^n}-V^{t}_{\cM^n})\\
    &\leq \frac{1}{1-\gamma}\bigg(\sum_{t=0}^{T-1} \mathbb{E}_{(s,a)\sim\tilde{d}_{\cM^n}}\left[\widehat{A}^t_{b^n}(s,a)\indict\{s\in\cK^n\}\right] + \sum_{t=0}^{T-1} \mathbb{E}_{(s,a)\sim\tilde{d}_{\cM^n}}\left[A^{t}_{b^n}(s,a) - \widehat{A}^t_{b^n}(s,a))\indict\{s\in\cK^n\}\right]\bigg)\\
    &\leq \frac{1}{1-\gamma}\bigg(12W\sqrt{\log(|\cA|)T} + \sum_{t=0}^{T-1} \mathbb{E}_{(s,a)\sim\tilde{d}_{\cM^n}}\left[\big(A^{t}_{b^n}(s,a)-\widehat{A}^t_{b^n}(s,a)\big) \indict\{s\in\cK^n\}\right]\bigg).
\end{align}
\end{proof}

Since the definition of $\widehat{A}^t_{b^n}$ is the same as the one for SPI-SAMPLE, Lemma \ref{lemma:SPI-I-Approx} and Lemma \ref{lemma:stat} are directly applied. In total, we have the following theorem for the sample complexity of ENIAC-SPI-COMPUTE.

\begin{theorem}[Main Result: Sample Complexity of ENIAC-SPI-COMPUTE]\label{thm:SPI-COMPUTE} Let $\delta\in(0, 1)$ and $\varepsilon\in(0, 1/(1-\gamma))$. With Assumptions \ref{ass:SPI-boundedtransfer}, \ref{ass:SPI-quadra}, \ref{ass:W}, and \ref{ass:spi-cover}, we set the hyperparameters as:
\begin{align}
    \beta &=\frac{\varepsilon(1-\gamma)}{2},
    ~T=\frac{144W^2\cdot\log|\cA|}{\varepsilon^2(1-\gamma)^2},
    ~N\geq \frac{384W^3|\cA|\log(|\cA|)\cdot\de(\cF, \beta)}{\varepsilon^4(1-\gamma)^4},
    ~\eta = \sqrt{\frac{\log(|\cA|)}{16W^2T}},\\
    \alpha &= \frac{1}{1+\sqrt{T}},
    ~\epsilon_1 = \frac{(1-\gamma)^4\varepsilon^4}{1536W^2|\cA|\log(|\cA|)\cdot\de(\cF, \beta)},
    ~\epsilon_2 = \frac{(1-\gamma)^4\varepsilon^4}{1248C W^3|\cA|\log(|\cA|)\de(\cF, \beta)},\\
    K &= \frac{1536W^3|\cA|^2(\log(|\cA|))^2\cdot \de(\cF, \beta)\cdot\big(\log(\frac{3NT\cdot\cN(\cF, \epsilon_1)}{\delta})\big)^2\cdot\log(\frac{6NT}{\delta})}{\varepsilon^4(1-\gamma)^4},\\
    M &= \frac{48000C^2W^5|\cA|\log(|\cA|)\de(\cF, \beta)\log(\frac{3NT\cdot\cN(\cF, \epsilon_2)}{\delta})}{\varepsilon^4(1-\gamma)^4},
\end{align}
and $\epsilon$ satisfies Equation \eqref{eq:epsilon} correspondingly. Then with probability at least $1-\delta$, for the average policy $\pi^N_{\text{ave}}:=\unif(\pi^2, \dots, \pi^{N+1})$, we have
\begin{align}
V^{\pi^N_{\text{ave}}} \geq V^{\tilde{\pi}} - \frac{4\sqrt{|\cA|\epsilon_{\text{bias}}}}{1-\gamma}-\epsilon_0\cdot \frac{200CW\cdot |\cA|\log(|\cA|)\cdot\de(\cF, \beta)}{\varepsilon^3(1-\gamma)^4} - 9\varepsilon
\end{align}
for any comparator $\tilde{\pi}$ with total number of samples:
\begin{align}
  \widetilde{O}\Big( \frac{C^2W^{10}\cdot|\cA|^2\cdot\big(\de(\cF, \beta)\big)^2\cdot \big(\log(\cN(\cF,  \epsilon'))\big)^2}{\varepsilon^{10}(1-\gamma)^{10}}\Big),
\end{align}
where $\epsilon'=\min(\epsilon_1, \epsilon_2)$. 
\end{theorem}

\begin{corollary}
If Assumption \ref{ass:eval_close} holds, with proper hyperparameters, the average policy $\pi^N_{\text{ave}}:=\unif(\pi^2,\dots,\pi^{N+1})$ of ENIAC-SPI-COMPUTE achieves $V^{\pi^N_{\text{ave}}}\geq V^{\tilde{\pi}} - \varepsilon$ with probability at least $1-\delta$ and total number of samples:
\begin{align}
    \widetilde{O}\Big(\frac{W^{10}\cdot|\cA|^2\cdot\big(\de(\cF, \beta)\big)^2\cdot \big(\log(\cN(\cF,  \epsilon'))\big)^2}{\varepsilon^{10}(1-\gamma)^{10}}\Big).
\end{align}
\end{corollary}

\section{Analysis of ENIAC-NPG}\label{app:NPG}
In this section, we provide the sample complexity of ENIAC-NPG-SAMPLE. For ENIAC-NPG-COMPUTE, it can be adapted from ENIAC-SPI-COMPUTE and ENIAC-NPG-SAMPLE.

The analysis of ENIAC-NPG-SAMPLE is in parallel to that of ENIAC-SPI-SAMPLE. As before, we provide a general result which considers model misspecification and Theorem \ref{thm:NPG-Sample} falls as a special case under the closedness Assumption \ref{ass:eval_close_npg}.


We simplify the notation as $\pi_{\theta}$ for $\pi_{f_{\theta}}(a|s):=\frac{\exp(f_{\theta}(s,a))}{\sum_{a'}\exp(f_{\theta}(s,a'))}$. Then for epoch $n$ iteration $t$ in ENIAC-NPG-SAMPLE,
\begin{align}
   \pi^n_t(\cdot|s)=\begin{cases}\pi_{\theta^n_t}(\cdot|s), & s\in\cK^n\\
   \unif(\{a\in\cA: (s,a)\notin\cK^n\}), & o.w.
   \end{cases}
\end{align}
We state the following assumptions to quantify the misspecification error.

\begin{assumption}[Bounded Transfer Error]\label{ass:NPG-boundedtransfer}
Given a target function $g:\cS\times\cA\rightarrow\RR$, we define the critic loss function $L(u;d,g,\pi_{\theta})$ with $d\in \Delta(\cS\times\cA)$ as:
\begin{align}
    L(u;d,g,\pi_{\theta}):= \mathbb{E}_{(s,a)\sim d} \left[(u^\top \nabla_{\theta}\log \pi_{\theta} - g)^2\right].
\end{align}
For the fixed comparator policy $\tilde{\pi}$ as mentioned in Section \ref{sec:proofdef}, we define a state-action distribution $\tilde{d}(s,a):=d^{\tilde{\pi}}_{s_0}(s)\circ \unif(\cA)$. In ENIAC-NPG-SAMPLE, for every epoch $n\in[N]$ and every iteration $t$ inside epoch $n$, we assume that
\begin{align}
    \inf_{u\in \cU^n_t} L(u; \tilde{d}, A^{t}_{b^n}-\bar{b}^n_t, \pi_{\theta^n_t})\leq \epsilon_{\text{bias}},
\end{align}
where $\cU_t^n:=\argmin_{u\in\cU} L(u;\rho^n_{\text{cov}}, A^{t}_{b^n}-\bar{b}^n_t, \pi_{\theta^n_t})$ and $\epsilon_{\text{bias}}\geq 0$ is a problem-dependent constant.
\cut{\begin{align}
    \tilde{u}_t^n\in\argmin_{u\in\cU} L(u;\rho^n_{\text{cov}}, A^{t}_{b^n}-\bar{b}^n_t, \pi_{\theta^n_t}),
\end{align}
where 
$\bar{b}^n_t(s,a):=b^n(s,a)-\mathbb{E}_{a'\sim\pi_t^n(\cdot|s)}b^n(s,a')$.
Then we assume that $\tilde{u}^n_t$ has a bounded prediction error when transferred from $\rho^n_{\text{cov}}$ to $\tilde{d}$, i.e.,
\begin{align}
    L(\tilde{u}_t^n; \tilde{d}, A^{t}_{b^n}-\bar{b}^n_t, \pi_{\theta^n_t})\leq \epsilon_{\text{bias}},
\end{align}}
\end{assumption}
Recall that
$\big(A^t_{b^n}-\bar{b}^n_t\big)(s,a)=Q^t_{b^n}(s,a)-b^n(s,a)-\mathbb{E}_{a\sim\pi^n_t(\cdot|s)}[Q^t_{b^n}(s,a)-b^n(s,a)]$. As before, we denote by $\tilde{u}^n_t$ a particular vector in $\cU^n_t$ such that $L(\tilde{u}_t^n; \tilde{d}, A^{t}_{b^n}-\bar{b}^n_t, \pi_{\theta^n_t})\leq 2\epsilon_{\text{bias}}$.
Note that we use $\nabla_\theta \log\pi_{\theta_t^n}$ as the linear features for critic fit at iteration $t$ epoch $n$, even though $\pi_t^n$ is not the same as $\pi_{\theta_t^n}$. Nevertheless, we show later that this choice of features is sufficient for good critic fitting on the known states, where we measure our critic error.

\cut{\begin{assumption}\label{ass:NPG-quadra}
For the same loss function $L$ as defined in Assumption \ref{ass:NPG-boundedtransfer}, we assume there exists some $C\geq 0$ and $\epsilon_0\geq 0$ such that
\begin{align}
    &\mathbb{E}_{(s,a)\sim\rho^n_{\text{cov}}} \left[\big((u^n_t)^\top \nabla_{\theta}
\log\pi_{\theta^n_t}-(\tilde{u}^n_t)^\top \nabla_\theta\log \pi_{\theta^n_t}\big)^2\right]\\
\leq &C\cdot\big(L(u^n_t; \rho^n_{\text{cov}}, A^{t}_{b^n}-\bar{b}^n_t, \pi_{\theta^n_t}) - L(\tilde{u}^n_t; \rho^n_{\text{cov}}, A^{t}_{b^n}-\bar{b}^n_t, \pi_{\theta^n_t})\big) + \epsilon_0.
\end{align}
\end{assumption}
}
\begin{remark}\label{rmk:npg}
Under the closedness condition Assumption \ref{ass:eval_close_npg},
\begin{align}
    A^t_{b^n}(s,a)-\bar{b}^n(s,a) &= Q^t_{b^n}(s,a)-b^n(s,a) - \mathbb{E}_{a'\sim\pi^n_t}(Q^t_{b^n}-b^n(s,a'))\\
    &=\mathbb{E}^{\pi^n_t}[r(s,a)+\gamma Q^t_{b^n}(s',a')]-\mathbb{E}_{a'\sim\pi^n_t}[\mathbb{E}^{\pi^n_t}[r(s,a')+\gamma Q^t_{b^n}(s'',a'')]]\\
    &\in\cG_{f_{\theta^n_t}},
\end{align}
where the last step follows, since $\pi_t^n$ can be described as $\pi_{\theta_t^n,\cK^n}$ under the notation of Assumption~\ref{ass:eval_close_npg}, whence the containment of $\cG_{f_{\theta^n_t}}$ follows.
Thus, there exists a vector $u\in\cU$ such that $u^\top \nabla \log \pi_{f_{\theta^n_t}}=
A^t_{b^n}-\bar{b}^n$ everywhere. We can then take $\epsilon_{\text{bias}}$ as 0 and $\tilde{u}^n_t=u$.
Assumption \ref{ass:NPG-boundedtransfer} therefore is a generalized version of the closedness condition.
\end{remark}


For NPG, the loss function $L$ is convex in the parameters $u$ since the features are fixed for every individual iteration. As a result, we naturally have an inequality as in Assumption \ref{ass:SPI-quadra} for SPI. We present it in the lemma below, which essentially follows a similar result for the linear case in \citet{agarwal2020pc}.

\begin{lemma}\label{lemma:NPG-quadra}
For the same loss function $L$ as defined in Assumption \ref{ass:NPG-boundedtransfer}, it holds that
\begin{align}
    &\mathbb{E}_{(s,a)\sim\rho^n_{\text{cov}}} \left[\big((u^n_t - \tilde{u}^n_t)^\top \nabla_\theta\log \pi_{\theta^n_t}\big)^2\right]\\
\leq &L(u^n_t; \rho^n_{\text{cov}}, A^{t}_{b^n}-\bar{b}^n_t, \pi_{\theta^n_t}) - L(\tilde{u}^n_t; \rho^n_{\text{cov}}, A^{t}_{b^n}-\bar{b}^n_t, \pi_{\theta^n_t}).
\end{align}
\end{lemma}
\begin{proof}
For the left-hand side, we have that
\begin{align}
    & \mathbb{E}_{(s,a)\sim\rho^n_{\text{cov}}} \left[\big((u^n_t)^\top \nabla_{\theta}
\log\pi_{\theta^n_t}-(\tilde{u}^n_t)^\top \nabla_\theta\log \pi_{\theta^n_t}\big)^2\right]\\
= & \mathbb{E}_{(s,a)\sim \rho^n_{\text{cov}}} \left[
\Big((u^n_t)^\top \nabla_{\theta}
\log\pi_{\theta^n_t} + \bar{b}^n_t - A^t_{b^n}\Big)^2\right] -\mathbb{E}_{(s,a)\sim\rho^n_{\text{cov}}} \left[\Big((\tilde{u}^n_t)^\top \nabla_\theta\log \pi_{\theta^n_t} + \bar{b}^n_t - A^t_{b^n}\Big)^2
\right]\\
& - 2\mathbb{E}_{(s,a)\sim\rho^n_{\text{cov}}} \left[ \Big((u^n_t)^\top \nabla_{\theta}
\log\pi_{\theta^n_t}-(\tilde{u}^n_t)^\top \nabla_\theta\log \pi_{\theta^n_t}\Big)\cdot \Big((\tilde{u}^n_t)^\top \nabla_\theta\log \pi_{\theta^n_t} + \bar{b}^n_t - A^t_{b^n}\Big)     \right]
\end{align}
Since $\tilde{u}^n_t$ is a minimizer. By first-order optimality condition, the cross term is greater or equal to 0. The desired result is obtained.
\end{proof}

\subsection{Sample Complexity of ENIAC-NPG-SAMPLE}
We follow the same steps as listed in \ref{sec:proofidea} and start with the bonus bound.

\begin{lemma}[NPG-SAMPLE: The Bound of Bonus]\label{lemma:NPG-I-Bonus}
With probability at least $1-N\delta$,
\begin{align}
    \sum_{n=1}^N V^{{\pi}^{n+1}}_{b^n} - V^{\pi^{n+1}} &\leq  \frac{2\epsilon^2+32G^2B^2K+\beta^2}{(1-\gamma)\beta^2K}\cdot\de(\cG_{\cF}, \beta)+ \frac{N}{1-\gamma}\sqrt{\frac{\log(2/\delta)}{2K}}.
\end{align}
\end{lemma}
The proof is similar to Lemma \ref{lemma:SPI-I-Bonus}. The only thing changed is the function approximation space. Thus we have $\de(\cG_\cF,\beta)$ instead of $\de(\cF,\beta)$ and $\|g^u_{\theta}\|_{\infty}\leq 2GB$, $\forall g^u_{\theta}\in\cG_{\cF}$.

Next, we establish the convergence result of NPG update. We focus on a specific episode $n$
and for each iteration $t$, we define
\begin{align}\label{eq:npg-hatA}
 \widehat{A}^t_{b^n}(s,a) &:= u_t^\top\nabla f_{\theta_t}(s,a) + b^n - \mathbb{E}_{a'\sim\pi_{\theta_t}(\cdot|s)}[u_t^\top\nabla f_{\theta_t}(s,a') + b^n(s,a')].
\end{align}
Since $\pi_t(\cdot|s)=\pi_{\theta_t}(\cdot|s)$ for  $s\in\cK^n$, $\mathbb{E}_{a'\sim \pi_t(\cdot|s)}[\widehat{A}^t_{b^n}(s,a')] = 0$ for $s\in\cK^n$.

From the algorithm we can see that $\widehat{A}^t_{b^n}$ is indeed our approximation to the real advantages $A^t_{b^n}$. In contrary to ENIAC-SPI, the actor update in ENIAC-NPG does not use $\widehat{A}^t_{b^n}$ directly but by modifying the parameter $\theta$. In the next lemma, we show how to link the NPG update to a formula of $\widehat{A}^t_{b^n}$ and eventually are able to bound the policy sub-optimality with function approximation error.
\cut{\begin{align}
    \theta_{t+1} = \theta_t + \eta u_t, \quad
    \pi_{t+1}(\cdot|s) \propto \exp \big(\widehat{A}^t_{b^n}(s,a)\cdot\indict\{s\in\cK^n\}\big)
\end{align}}
\begin{lemma}[NPG-SAMPLE: Convergence]\label{lemma:NPG-I-Conv}
In ENIAC-NPG-SAMPLE, let $\widehat{A}^t_{b^n}$ be as defined in Equation \eqref{eq:npg-hatA} and $\eta = \sqrt{\frac{\log(|\cA|)}{(16D^2+\Lambda B^2)T}}$. For any epoch $n\in[N]$, NPG-SAMPLE obtains a sequence of policies $\{\pi_t\}_{t=0}^{T-1}$ such that when comparing to $\tilde{\pi}$:
\begin{align}
    & \frac{1}{T}\sum_{t=0}^{T-1}(V_{\cM^n}^{\tilde{\pi}^n}-V^t_{b^n})=\frac{1}{T}\sum_{t=0}^{T-1}(V_{\cM^n}^{\tilde{\pi}^n}-V^{t}_{\cM^n})\\
    &\leq\frac{1}{1-\gamma}\bigg(2\sqrt{\frac{\log(|\cA|)(16D^2 + \Lambda B^2)}{T}}  + \frac{1}{T}\sum_{t=0}^{T-1} \mathbb{E}_{(s,a)\sim\tilde{d}_{\cM^n}}\left[\big(A^t_{b^n}(s,a)-\widehat{A}^t_{b^n}(s,a)\big) \indict\{s\in\cK^n\}\right]\bigg).
\end{align}
\end{lemma}
\begin{proof}
For the same reason as in Lemma \ref{lemma:SPI-I-Conv}, we have 
\cut{\begin{align}
    A^t_{\cM^n}(s,a^\dagger) = Q^t_{\cM^n}(s,a^\dagger) - V^t_{\cM^n}(s) = 1 - (1-\gamma)\cdot V^t_{\cM^n}(s) \leq 0, ~~\forall s\notin \cK^n.
\end{align}
Based on the above inequality, it holds that}
\begin{align}\label{eq:NPG-II_v_1}
    V^{\tilde{\pi}^n}_{\cM^n} - V^t_{\cM^n} \leq \frac{1}{1-\gamma}\sum_{(s,a)}\tilde{d}_{\cM^n}(s,a)A^t_{b^n}(s,a)\indict\{s\in\cK^n\}.
\end{align}
We focus on on $s\in\cK^n$. Then
$\pi_t(\cdot|s)\propto \exp(f_{\theta_t}(s,\cdot))$ and $b(s,\cdot)=0$. It holds that
\begin{align}
    &\KL(\tilde{\pi}^n(\cdot|s), \pi_{t+1}(\cdot|s)) - \KL(\tilde{\pi}^n(\cdot|s), \pi_{t}(\cdot|s)) \\
    &= -\mathbb{E}_{a\sim\tilde{\pi}^n(\cdot|s)}\left[f_{\theta_{t+1}}(s,a) - f_{\theta_t}(s,a)\right] + \log \frac{\sum_a \exp (f_{\theta_{t+1}}(s,a))}{\sum_a \exp (f_{\theta_{t}}(s,a))} \\
    & \leq -\mathbb{E}_{a\sim\tilde{\pi}^n(\cdot|s)}[\eta\cdot u_t^\top\nabla_{\theta} {f_{\theta_t}}- \eta^2\frac{\Lambda B^2}{2}]+ \log\frac{\sum_{a}\exp(f_{\theta_t}(s,a)+\eta\cdot u_t^\top \nabla_{\theta}f_{\theta_t}+\eta^2\Lambda B^2/2)}{\sum_{a}\exp (f_{\theta_t}(s,a))}\\
    &=-\eta\cdot\mathbb{E}_{a\sim\tilde{\pi}^n(\cdot|s)}[\widehat{A}^t_{b^n}(s,a)] -\eta\cdot\mathbb{E}_{a'\sim \pi_{t}(\cdot|s)} u_t^\top\nabla_{\theta} f_{\theta_t}(s,a')\\
    &~~+\log \Big(\sum_a\pi_t(s,a)\exp\big(\eta\cdot \widehat{A}^t_{b^n}(s,a)+\eta\cdot\mathbb{E}_{a'\sim \pi_t(\cdot|s)} u_t^\top\nabla_{\theta} f_{\theta_t}\big)\Big)+\eta^2\Lambda B^2\\
    &= -\mathbb{E}_{a\sim\tilde{\pi}^n(\cdot|s)}[\eta\widehat{A}^t_{b^n}(s,a)] + \log \Big(\sum_a \pi_t(a|s)\exp\big(\eta\widehat{A}^t_{b^n}(s,a) \big)\Big)  + \eta^2\Lambda B^2.
\end{align}
where the inequality is by Taylor expansion and the regularity assumption \ref{ass:NPG-REG}:
\begin{align}
    f_{\theta_t}+(\theta_{t+1}-\theta_t)^\top \nabla_{\theta} f_{\theta_t} - \frac{\Lambda}{2} \|\theta_{t+1}-\theta_t\|_2^2\leq f_{\theta_{t+1}}\leq f_{\theta_t}+(\theta_{t+1}-\theta_t)^\top \nabla_{\theta} f_{\theta_t} + \frac{\Lambda}{2} \|\theta_{t+1}-\theta_t\|_2^2.
\end{align}
Since $|\widehat{A}^t_{b^n}(s,a)|\leq 4D$ and $\eta\leq 1/(4D)$ when $T>\log(|\cA|)$,  $\eta\widehat{A}^t_{b^n}(s,a)\leq 1$. By the inequality that $\exp(x)\leq 1+x+x^2$ for $x\leq 1$, we have that
\begin{align}
   &\log \Big(\sum_a \pi_t(a|s)\exp\big(\eta\widehat{A}^t_{b^n}(s,a) \big)\Big)\\
   \leq &\log \Big(1+\mathbb{E}_{a\sim \pi_t(\cdot|s)}[\eta\widehat{A}^t_{b^n}(s,a)]+16\eta^2D^2\Big)
   \leq 16\eta^2D^2.
\end{align}
Hence, for $s\in\cK^n$,
\begin{align}
\KL(\tilde{\pi}^n(\cdot|s), \pi_{t+1}(\cdot|s)) - \KL(\tilde{\pi}^n(\cdot|s), \pi_{t}(\cdot|s)) \leq -\eta\mathbb{E}_{a\sim\tilde{\pi}^n(\cdot|s)}[\widehat{A}^t_{b^n}(s,a)] + \eta^2(16D^2 + \Lambda B^2).
\end{align}
Adding both sides from $t=0$ to $T-1$ and taking $\eta = \sqrt{\frac{\log(|\cA|)}{(16D^2+\Lambda B^2)T}}$, we get
\begin{align}
    &\sum_{t=0}^{T-1} \mathbb{E}_{(s,a)\sim\tilde{d}_{\cM^n}}\left[\widehat{A}^t_{b^n}(s,a)\indict\{s\in\cK^n\}\right]\\
    \leq &\frac{1}{\eta}\mathbb{E}_{s\sim\tilde{d}_{\cM^n}}\left[\big(\KL(\tilde{\pi}^n(\cdot|s), \pi_{0}(\cdot|s)) - \KL(\tilde{\pi}^n(\cdot|s), \pi_{T}(\cdot|s))\big)\indict\{s\in\cK^n\}\right] + \eta T(16D^2+ \Lambda B^2) \\
    \leq &\log(|\cA|)/\eta + \eta T(16D^2 + \Lambda B^2) \leq 2\sqrt{\log(|\cA|)\cdot(16D^2 + \Lambda B^2)\cdot T}.
\end{align}
Combining with Equation \eqref{eq:NPG-II_v_1}, the regret on $\cM^n$ satisfies
\begin{align}
    &\sum_{t=0}^{T-1} (V^{\tilde{\pi}^n}_{\cM^n}-V^t_{\cM^n})\\
    &\leq \frac{1}{1-\gamma}\sum_{t=0}^{T-1} \mathbb{E}_{(s,a)\sim\tilde{d}_{\cM^n}}\left[\widehat{A}^t_{b^n}(s,a)\indict\{s\in\cK^n\}\right] + \frac{1}{1-\gamma}\sum_{t=0}^{T-1} \mathbb{E}_{(s,a)\sim\tilde{d}_{\cM^n}}\left[A^t_{b^n}(s,a) - \widehat{A}^t_{b^n}(s,a))\indict\{s\in\cK^n\right]\\
    &\leq \frac{1}{1-\gamma}\bigg(2\sqrt{\log(|\cA|)(16D^2 + \Lambda B^2)T} + \sum_{t=0}^{T-1} \mathbb{E}_{(s,a)\sim\tilde{d}_{\cM^n}}\left[\big(A^t_{b^n}(s,a)-\widehat{A}^t_{b^n}(s,a)\big) \indict\{s\in\cK^n\}\right]\bigg).
\end{align}
\end{proof}

Next, we establish two lemmas to bound the difference between the true advantage $A^t_{b^n}(s,a)$ and the approximation $\widehat{A}^t_{b^n}(s,a)$ .
\begin{lemma}[Approximation Bound]\label{lemma:NPG-I-Approx}
At epoch $n$, assume for all $0\leq t\leq T-1$,
\cut{we denote a best on-policy fit as
\begin{align}
    u^*_t&\in\argmin_{u\in\cU}L(u; \rho^n_{\text{cov}}, A^t_{b^n}-\bar{b}^n_t, \pi_{\theta_t}) = \mathbb{E}_{(s,a)\sim \rho^n_{\text{cov}}}\left[(A^t_{b^n} -\bar{b}^n_t - u^\top \nabla_{\theta}\log \pi_{\theta_t})^2\right].
\end{align}}
\begin{align}
   L(u^n_t; \rho^n_{\text{cov}}, A^t_{b^n}-\bar{b}^n_t, \pi_{\theta^n_t})\leq L(\tilde{u}_t^n; \rho^n_{\text{cov}}, A^t_{b^n}-\bar{b}^n_t, \pi_{\theta^n_t}) + \epsilon_{\text{stat}},
\end{align}
where $\epsilon_{\text{stat}}>0$ is to be determined later, and
\begin{align}\label{eq:npg-epsilon}
\epsilon^2 = NK\big(\epsilon_{\text{stat}} + 16D\epsilon_1\big) + 8D^2\log(\cN(\cG_\cF, \epsilon_1)/\delta)\cdot\sqrt{NK},
\end{align}
where $\epsilon$ is used in bonus function design (see Section \ref{sec:bonus}) and $\epsilon_1$ is to be determined. Under Assumption  \ref{ass:NPG-boundedtransfer} and \ref{ass:NPG-REG}, we have that for every $0\leq t\leq T-1$, with probability at least $1-(n+1)\delta$,
\begin{align}
    \mathbb{E}_{(s,a)\sim \tilde{d}_{\cM^n}}\Big(A^t_{b^n}(s,a) - \widehat{A}^t_{b^n}(s,a)\Big) \leq 4\sqrt{|\cA|\epsilon_{\text{bias}}} + 2\beta.
\end{align}
\end{lemma}

\begin{lemma}\label{lemma:npgstat}
Following the same notation as in Lemma \ref{lemma:NPG-I-Approx}, it holds with probability at least $1-\delta$ that
\begin{align}
    L(u^n_t;\rho^n_{\text{cov}},A^t_{b^n}-\bar{b}^n_t, \pi_{\theta^n_t})- L(\tilde{u}_t^n;\rho^n_{\text{cov}}, A^t_{b^n}-\bar{b}^n_t, \pi_{\theta^n_t}) 
    &\leq \frac{500D^4\cdot d\log\big( \frac{6D}{\epsilon_2\delta}\big)}{M} + 13D^2\cdot \epsilon_2,
\end{align}
where $d$ is the linear dimension of $u$.
\end{lemma}
The proofs of the above lemmas can be easily adapted from Lemma \ref{lemma:SPI-I-Approx} or Lemma \ref{lemma:stat} by replacing $f_t$ with $u_t^\top \nabla f_{\theta_t}$, $\tilde{f}^n_t$ with $({\tilde{u}_t^n})^\top\nabla f_{\theta_t}$, and $\cF$ with $\cG_\cF$. In particular, for Lemma \ref{lemma:npgstat}, since the linear feature is fixed for critic fit at iteration $t$ epoch $n$, the function cover is defined on the space $\cG_{f_{\theta^n_t}}$. By Lemma \ref{lemma:linearcover}, the covering number is therefore represented with the linear dimension of $u$, $d$.

In the following, we present the detailed form of the sample complexity of NPG-SAMPLE.
\begin{theorem}[Main Result: Sample Complexity of ENIAC-NPG-SAMPLE]\label{thm:NPG-SAMPLE-DETAIL} Let $\delta\in(0, 1)$ and $\varepsilon\in(0, 1/(1-\gamma))$. With Assumptions \ref{ass:NPG-boundedtransfer} and \ref{ass:NPG-REG}, we set the hyperparameters as:
\begin{align}
    \beta &=\frac{\varepsilon(1-\gamma)}{2},
    T=\frac{64(D^2+\Lambda B^2)\cdot\log|\cA|}{\varepsilon^2(1-\gamma)^2},
    N\geq \frac{128B^2G^2\cdot\de(\cG_{\cF}, \beta)}{\varepsilon^3(1-\gamma)^3},
   \eta = \sqrt{\frac{\log(|\cA|)}{(16D^2+\Lambda B^2)T}}\\
    \epsilon_1 &= \frac{(1-\gamma)^3\varepsilon^3}{128D\cdot\de(\cG_\cF, \beta)},
    ~~K = \frac{32D^2\cdot \de(\cG_\cF, \beta)\cdot\big(\log(\frac{3NT\cdot\cN(\cG_\cF, \epsilon_1)}{\delta})\big)^2\cdot\log(\frac{6NT}{\delta})}{\varepsilon^3(1-\gamma)^3},\\
    ~~\epsilon_2 &= \frac{(1-\gamma)^3\varepsilon^3}{110 D^2\cdot\de(\cG_\cF, \beta)},
    ~~M = \frac{4000D^4\cdot \de(\cG_\cF, \beta)\cdot d\log(\frac{18DNT}{\epsilon_2\delta})}{\varepsilon^3(1-\gamma)^3},
\end{align}
and $\epsilon$ satisfies Equation \eqref{eq:npg-epsilon} correspondingly. Then with probability at least $1-\delta$, for the average policy $\pi^N_{\text{ave}}:=\unif(\pi^2, \dots, \pi^{N+1})$, we have
\begin{align}\label{eq:main}
V^{\pi^N_{\text{ave}}} \geq V^{\tilde{\pi}} - \frac{4\sqrt{|\cA|\epsilon_{\text{bias}}}}{1-\gamma}- 9\varepsilon
\end{align}
for any comparator $\tilde{\pi}$ with total number of samples:
\begin{align}
  \widetilde{O}\Big( \frac{D^6(D^2+\Lambda B^2)\cdot\big(\de(\cG_{\cF}, \beta)\big)^2\cdot \big(\log(\cN(\cG_\cF, \epsilon'))\big)^2}{\varepsilon^{8}(1-\gamma)^{8}}\Big),
\end{align}
where $\epsilon'=\min(\epsilon_1, \epsilon_2)$ such that $\log(\cN(\cG_\cF, \epsilon'))=\Omega(d)$.

\end{theorem}
The proof is similar to that of Theorem \ref{thm:SPI-SAMPLE-Detail}. We also have the following result when the closedness assumption is satisfied.
\begin{corollary}
If Assumption \ref{ass:eval_close_npg} holds, with proper hyperparameters, the average policy $\pi^N_{\text{ave}}:=\unif(\pi^2,\dots,\pi^{N+1})$ of ENIAC-NPG-SAMPLE achieves $V^{\pi^N_{\text{ave}}}\geq V^{\tilde{\pi}} - \varepsilon$ with probability at least $1-\delta$ and total number of samples:
\begin{align}
    \widetilde{O}\Big( \frac{D^6(D^2+\Lambda B^2)\cdot\big(\de(\cG_{\cF}, \beta)\big)^2\cdot \big(\log(\cN(\cG_\cF, \epsilon'))\big)^2}{\varepsilon^{8}(1-\gamma)^{8}}\Big)
\end{align}
\end{corollary}
Note that under Assumption \ref{ass:eval_close_npg}, as mentioned in Remark \ref{rmk:npg},  $\epsilon_{\text{bias}}=0$.

\section{Auxiliary Lemmas}
\begin{lemma}\label{lemma:cover}
Given a function class $\cF$, for its covering number, we have $\cN(\Delta \cF, \epsilon)\leq \cN(\cF, \epsilon/2)^2$.
\end{lemma}
\begin{proof}
Let $\Delta \cC(\cF,\epsilon/2):=\{f-f'|f,f'\in\cC(\cF,\epsilon/2)\}$. Then $\Delta \cC(\cF, \epsilon/2)$ is an $\epsilon$-cover for $\Delta \cF$ and $|\Delta \cC(\cF, \epsilon/2)| \leq |\cC(\cF, \epsilon/2)|^2 \leq \cN(\cF, \epsilon/2)^2$. 
\end{proof}

\begin{lemma}\label{lemma:linearcover}
Given $f\in\cF$, under the regularity Assumption \ref{ass:NPG-REG}, we have that the covering number of the linear class $\cG_f:=\{u^\top\nabla_{\theta}\log\pi_f, u\in\cU\subset\RR^d, f\in\cF\}$ achieves $\cN(\cG_f, \epsilon)\leq\big(\frac{3D}{\epsilon}\big)^d$.
\end{lemma}
\begin{proof}
In order to construct a cover set of $\cG_f$ with radius $\epsilon_2$, we need that for any $u\in\cU\subset\RR^d$, there exist a $\tilde{u}$, such that
\begin{align}
    \|u^\top\nabla_{\theta}\log\pi_f(s,a)-\tilde{u}^\top\nabla_{\theta}\log\pi_f(s,a)\|_{\infty}\leq\epsilon_2.
\end{align}
where the infinity norm is taken over all $(s,a)\in\cS\times\cA$. By Cauchy-Schwarz inequality, we have
\begin{align}
   \|u^\top\nabla_{\theta}\log\pi_f-\tilde{u}^\top\nabla_{\theta}\log\pi_f\|_{\infty}&= \|(u-\tilde{u})^\top\nabla_{\theta}\log\pi_f\|_{\infty}\leq 2G\|u-\tilde{u}\|_2.
\end{align}
Thus, it is enough to have $\|u-\tilde{u}\|_2\leq\epsilon_2/(2G)$, which is equivalent to cover a ball in $\RR^d$ with radius $B$ (recall that $\|u\|\leq B$) with small balls of radius $\epsilon_2/(2G)$. The latter has a covering number bounded by $\Big(\frac{6BG}{\epsilon_2}\Big)^d\leq \Big(\frac{6D}{\epsilon_2}\Big)^d$\footnote{The covering number of Euclidean balls can be easily found in literature.}.
\end{proof}
\section{Algorithm Hyperparameters}\label{app:algo}
In this section, we present more details about the implementation in our experiments. All algorithms were based on the PPO implementation of \cite{shangtong2018modularized}. The network structure is described in the main body and the last layer outputs the parameters of a 1D Gaussian for action selection. 

\begin{algorithm}[htbp]
  \caption{Width Training in ENIAC
  }
  \label{alg:width}
  \begin{algorithmic}[1]
  \State \textbf{Input:} Replay buffer $\cZ^n$, query batch $\cZ^n_Q$.
  \State Initialize $f$ with the same network structure as the critic.
  \State Copy $f'$ as $f$ and fix $f'$ during training.
  \For {$i=1$ \textbf{to} $I$}
  \State Sample a minibatch $\cD_Q$ from $\cZ^n_Q$
  \For {$j=1$ \textbf{to} $J$}
  \State Sample a minibatch $\cD_j$ from $\cZ^n$
  \State Do one step of gradient descent on $f$ with loss in Equation \eqref{eq:width-loss-train} and $\cD_Q$ and $\cD_j$. 
  \EndFor
  \EndFor
  \State\textbf{Output:}  $w^n:=|f-f'|$
  \end{algorithmic}
\end{algorithm}

\begin{table}[htbp] 
\centering
\caption{ENIAC Width Training Hyperparameters}
\begin{tabular}{cccc} 
\toprule    
Hyperparameter & 2-layer & 4-layer & 6-layer \\    
\midrule 
$\lambda$ & 0.1 & 0.1 & 0.1\\
$\lambda_1$ & 0.01 &0.01 & 0.01\\
$|Z_Q|$ & 20000 & 20000 & 20000\\
Learning Rate & 0.001 & 0.001 & 0.0015\\ 
$|\cD_j|$ & 160 &160 & 160 \\
$|\cD_Q|$ & 20 & 20 & 10\\
Gradient Clippling &  5.0 & 5.0 & 5.0 \\
$I$ & 1000 & 1000 & 1000\\
$J$  & 10 & 10 & 10\\
\bottomrule   
\end{tabular}
\label{table:width}
\end{table}

The width training process is presented in Algorithm \ref{alg:width}. To stabilize training, for each iteration we sample a minibatch $\cD_Q$ from the query batch, then run several steps of stochastic gradient descent with changing minibatches on $\cZ^n$ while fixing $
\cD_Q$. The hyperparameters for width training are listed in Table \ref{table:width}. 

For PC-PG, we follow the same implementation as mentioned in \cite{agarwal2020pc}; for PPO-RND, the RND network has the same architecture as the policy network, except that the last linear layer mapping hidden units to actions is removed. We found that tuning the intrinsic reward coefficient was important for getting good performance for RND. The hyperparameters for optimization are listed in Table \ref{table:pcpg} and \ref{table:rnd}.

\begin{table}[htbp]
\centering
\caption{ENIAC/PC-PG Optimization Hyperparameters}
\begin{tabular}{ccccc}  
\toprule    
Hyperparameter & Values Considered & 2-layer & 4-layer & 6-layer \\    
\midrule   
Learning Rate & $e^{-3}, 5e^{-4}, e^{-4}$ & $5e^{-4}$ & $5e^{-4}$ & $5e^{-4}$\\   
$\tau_{\text{GAE}}$ & 0.95 & 0.95 & 0.95 & 0.95 \\
Gradient Clippling & 0.5, 1, 2, 5 & 5.0 & 5.0 & 5.0 \\  
Entropy Bonus & 0.01 & 0.01 & 0.01 & 0.01\\
PPO Ratio Clip & 0.2 & 0.2 & 0.2 & 0.2\\
PPO Minibatch & 160 & 160 & 160 & 160\\
PPO Optimization Epochs & 5 & 5 & 5 & 5\\
$\epsilon$-greedy sampling & 0, 0.01, 0.05 & 0.05 & 0.05 & 0.05\\
\bottomrule   
\end{tabular}
\label{table:pcpg} 
\end{table}

\begin{table}[htbp] 
\centering
\caption{PPO-RND Hyperparameters}
\begin{tabular}{ccccc} 
\toprule    
Hyperparameter & Values Considered & 2-layer & 4-layer & 6-layer \\    
\midrule   
Learning Rate & $e^{-3}, 5e^{-4}, e^{-4}$ & $e^{-4}$ & $e^{-4}$ & $e^{-4}$\\   
$\tau_{\text{GAE}}$ & 0.95 & 0.95 & 0.95 & 0.95 \\
Gradient Clippling & 5.0 & 5.0 & 5.0 & 5.0 \\  
Entropy Bonus & 0.01 & 0.01 & 0.01 & 0.01\\
PPO Ratio Clip & 0.2 & 0.2 & 0.2 & 0.2\\
PPO Minibatch & 160 & 160 & 160 & 160\\
PPO Optimization Epochs & 5 & 5 & 5 & 5\\
Intrinsic Reward Normalization & true, false & false & false & false\\
Intrinsic Reward Coefficient & 0.5, 1, $e, e^2,e^3, 5e^3, e^4$ & $5e^3$ & $e^3$ & $e^3$\\
\bottomrule   
\end{tabular}
\label{table:rnd}
\end{table}

\end{document}